\documentclass{article}

\usepackage{microtype}
\usepackage{graphicx}
\usepackage{subfigure}
\usepackage{booktabs} 

\usepackage{hyperref}
\usepackage{algpseudocode}

\usepackage[accepted]{icml2025}

\usepackage{amsmath}
\usepackage{amssymb}
\usepackage{mathtools}
\usepackage{amsthm}

\usepackage[capitalize,noabbrev]{cleveref}

\theoremstyle{plain}
\newtheorem{theorem}{Theorem}[section]
\newtheorem{proposition}[theorem]{Proposition}
\newtheorem{lemma}[theorem]{Lemma}

\newtheorem{properties}[theorem]{Properties}
                      
\theoremstyle{definition}
\newtheorem{definition}[theorem]{Definition}
\newtheorem{assumption}[theorem]{Assumption}
\theoremstyle{remark}

\usepackage[textsize=tiny]{todonotes}

\usepackage{multirow}
\usepackage{wasysym}
\usepackage{amssymb}
\usepackage{makecell}
\usepackage{stfloats}
\usepackage{tabularx}
\usepackage{caption}
\usepackage{bbm}
\usepackage{booktabs}
\usepackage{cleveref}
\usepackage{subcaption}

\usepackage{makecell}
\usepackage{float}

\usepackage{soul, color, xcolor}

\usepackage{amsmath}
\usepackage{multirow}

\usepackage{booktabs}

\usepackage{ltablex}

\icmltitlerunning{Learning Optimal Multimodal Information Bottleneck Representations}
           
\begin{document}

\twocolumn[
\icmltitle{Learning Optimal Multimodal Information Bottleneck Representations}




\icmlsetsymbol{equal}{*}

\begin{icmlauthorlist}
\icmlauthor{Qilong Wu}{equal,zuel_ts}
\icmlauthor{Yiyang Shao}{zuel_f}
\icmlauthor{Jun Wang}{equal,iwudao}
\icmlauthor{Xiaobo Sun}{equal,emory}
\end{icmlauthorlist}

\icmlaffiliation{zuel_ts}{School of Statistics and Mathematics, Zhongnan University of Economics and Law}
\icmlaffiliation{zuel_f}{School of Finance, Zhongnan University of Economics and Law}
\icmlaffiliation{iwudao}{iWudao}
\icmlaffiliation{emory}{School of Medicine, Department of Human Genetics, Emory University}


\icmlcorrespondingauthor{Xiaobo Sun}{xsun28@emory.edu}

\icmlkeywords{Machine Learning, ICML}

\vskip 0.3in
]



\printAffiliationsAndNotice{\icmlEqualContribution} 

\begin{abstract}
Leveraging high-quality joint representations from multimodal data can greatly enhance model performance in various machine-learning based applications. Recent multimodal learning methods, based on the multimodal information bottleneck (MIB) principle, aim to generate optimal MIB with maximal task-relevant information and minimal superfluous information via regularization. However, these methods often set  \textit{ad hoc} regularization weights and overlook imbalanced task-relevant information across modalities, limiting their ability to achieve optimal MIB. To address this gap, we propose a novel multimodal learning framework, Optimal Multimodal Information Bottleneck (OMIB), whose optimization objective guarantees the achievability of optimal MIB by setting the regularization weight within a theoretically derived bound. OMIB further addresses imbalanced task-relevant information by dynamically adjusting regularization weights per modality, promoting the inclusion of all task-relevant information. Moreover, we establish a solid information-theoretical foundation for OMIB's optimization and implement it under the variational approximation framework for computational efficiency. Finally, we empirically validate the OMIB’s theoretical properties on synthetic data and demonstrate its superiority over the state-of-the-art benchmark methods in various downstream tasks.
\end{abstract}

\section{Introduction}
In the parable "\textit{Blind men and an elephant}", a group of blind men attempts to perceive the elephant's shape through touch, but each inspects only a single, distinct part (e.g., tusk, leg). Consequently, they deliver conflicting descriptions, as their judgments are based solely on the part they touch.

In the context of machine learning, this parable underscores the significance of multimodal learning, which integrates and leverages multimodal data (akin to the elephant's body parts) to grasp a holistic understanding, thereby enhancing inference and prediction accuracy. In multimodal learning, unimodal features are extracted from each modalities and fused with various fusion strategies, such as tensor-based \cite{zadeh-etal-2017-tensor,liu-etal-2018-efficient-low}, attention-based \cite{guo2020ld,xiao2020multimodality,9968207}, and graph-based \cite{arun2022multimodal,huang2021temporal}, to generate multimodal embeddings. However, a major limitation of these methods is their potential to include superfluous and redundant information from each modality, increasing embedding complexity and the risk of overfitting \cite{9767641,wan2021multi}.

\begin{figure}[t]
\centering
\includegraphics[width=0.91\linewidth]{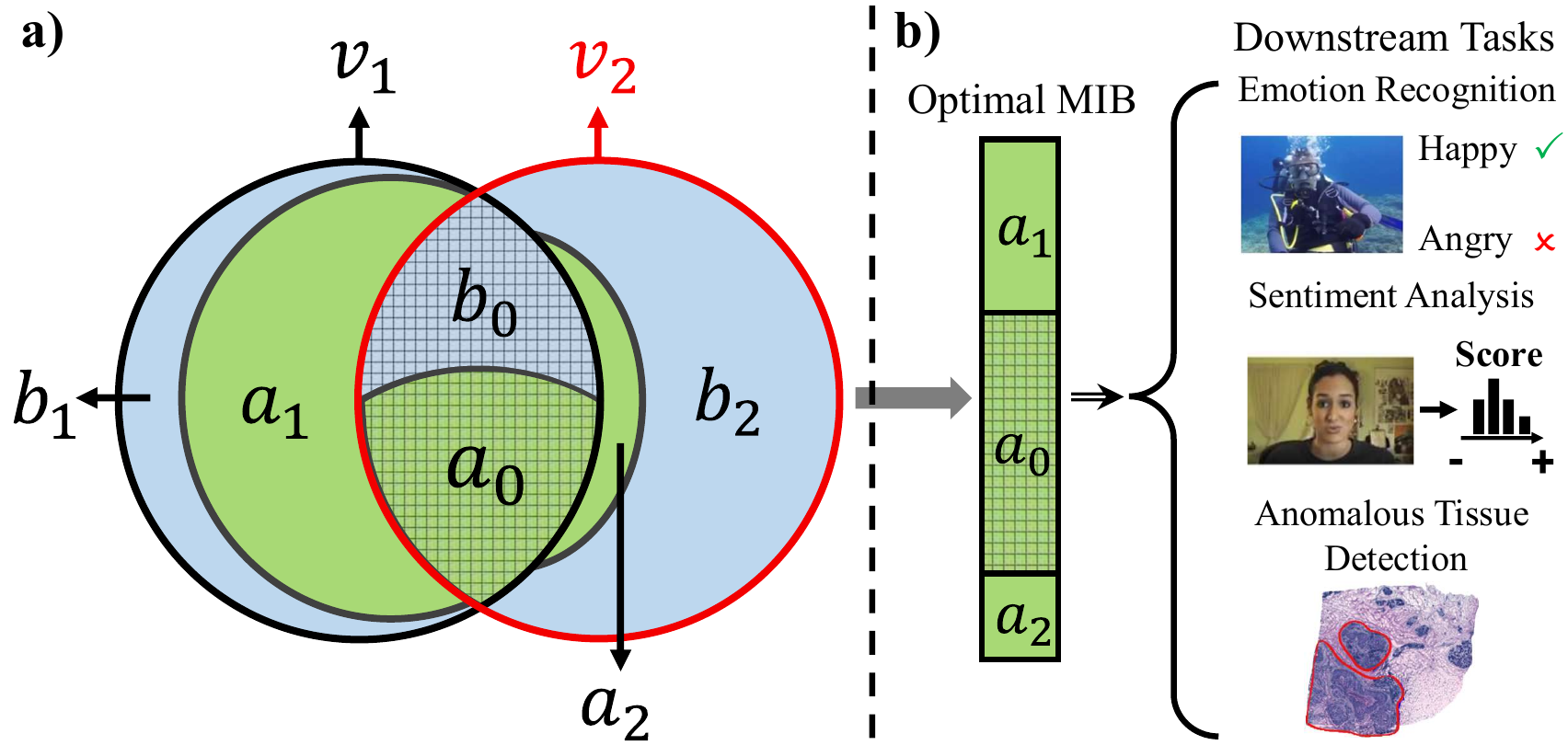}
\caption{
\textbf{a)} Venn diagrams for two data modalities ($v_1$ and $v_2$). The gridded area represents consistent information, while the non-gridded area denotes modality-specific information. Task-relevant information is highlighted in green, whereas superfluous information is shown in blue. \textbf{b)} An optimal MIB should exclusively contain task-relevant, non-superfluous information (i.e., $a_0, a_1$ and $a_2$) to be utilized in downstream tasks for enhanced performance. 
}
 \label{Fig: Information diagrams-1}
\end{figure}\par

\begin{figure}[t]
\centering
\includegraphics[width=0.91\linewidth]{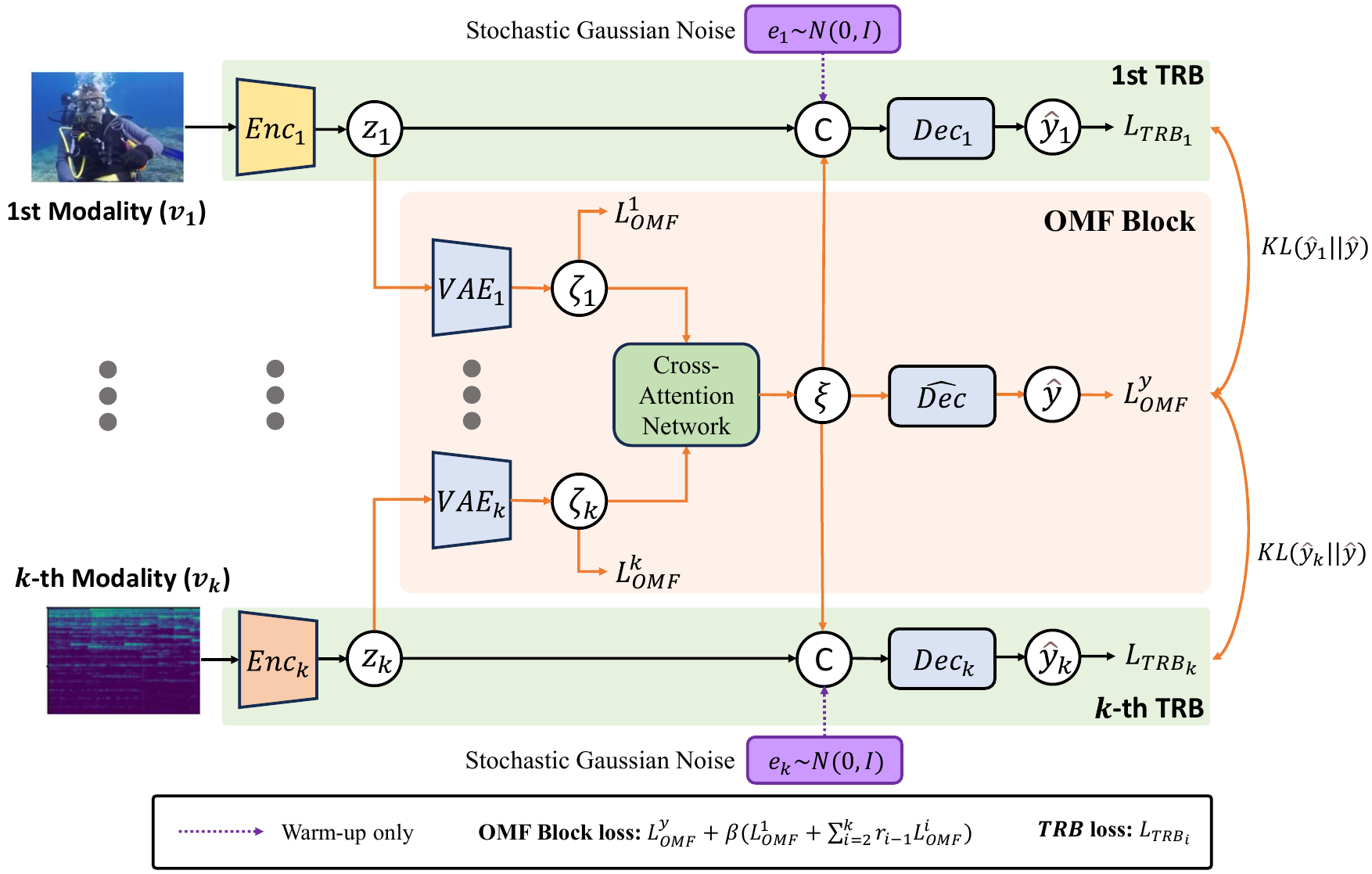}
\caption{\textbf{OMIB Framework.} Here, 'C' represents the concatenation operation. For the definitions of other notations, refer to the \Cref{sec:method} and \Cref{tab:notation}.}
\label{Fig: framework}
\end{figure}
\par

From an information theory perspective, a comprehensive multimodal learning method should account for five factors: \textbf{consistency} \cite{tian2021farewell}, \textbf{specificity} \cite{LIU2024101943}, \textbf{complementarity} \cite{wan2021multi},  \textbf{sufficiency} \cite{Federici2020Learning}, and \textbf{conciseness} (a.k.a. nonredundacy) \cite{wang2019deep}.  As illustrated in \Cref{Fig: Information diagrams-1}a, on the input side, \textit{consistency} describes information shared across input modalities (gridded area), while \textit{specificity} refers to the information unique to individual modalities (non-gridded area). Both consistent and modality-specific information may contain task-relevant (green area) or superfluous (gray area) components. \textit{Complementarity} pertains to modality-specific, task-relevant information ($a_1$ and $a_2$), enabling multimodal embeddings to surpass unimodal ones in downstream tasks. On the output side, an optimal multimodal embedding (as shown in \Cref{Fig: Information diagrams-1}b and \Cref{def:optimal MIB} below) must be \textit{sufficient}, capturing maximal task-relevant information—both consistent ($a_0$) and complementary ($a_1, a_2$)—across modalities. Meanwhile, it should be \textit{concise}, minimizing both cross-modality ($b_0$) and modality-specific superfluous information ($b_1,b_2$) to reduce complexity. This optimal multimodal embedding can then be applied to various downstream tasks, such as multimodal sentiment analysis \cite{9767641} and pathological tissue detection using histology and gene expression data \cite{xu2024detecting}), for enhanced performance.

To this end, multimodal learning methods based on the Multimodal Information Bottleneck (MIB) principle have been developed, which generally follow a common paradigm: modality-specific representations are extracted and fused into MIB via deep networks. The MIB is then optimized to balance two objectives: maximizing mutual information between the embeddings and task-relevant labels for sufficiency; and minimizing mutual information between the embeddings and the raw input to purge superfluous information and ensure conciseness \cite{wang2019deep,wan2021multi,9968207,Fang2024Dynamic}. This process is formalized as:
\begin{equation}\label{eq:paradigm}
\begin{split}
    &x_i=E_i(v_i),\\ 
    &z=F(x_1,x_2,...),\\
    &\mathcal{O}(v_i,z,y):=\max_{z}I(y;z) - \sum_{i}\beta_i I(v_i;z),
\end{split}
\end{equation}
where $E_i,F, I$, and $\mathcal{O}$ represent the modality-specific encoder, multimodal fusion function, mutual information function, and optimization objective, respectively. $v_i,x_i,z,$ and $y$ denote the raw data of the $i$-th modality, its extracted representation, the MIB, and task labels. Particularly,  $\beta_i$ serves as the regularization parameter for constraining superfluous information between the MIB and the $i$-th modality. 

Despite their promising performance, these methods face three major limitations.  First, the achievability of optimal MIB is not guaranteed. Since the regularization parameters (e.g., $\beta$s in \Cref{eq:paradigm}) control the trade-off between sufficiency and conciseness, their values are critical for optimizing MIB \cite{tian2021farewell}. If the value is too small, superfluous information may be retained, leading to suboptimal MIB. Conversely, if too large, task-relevant information may be excluded due to an overemphasis on conciseness, compromising MIB's sufficiency. However, existing MIB methods determine these parameter values in an \textit{ad hoc} manner, limiting their ability to achieve an optimal MIB. Second, an ideal MIB method should dynamically adjust regularization weights based on the remaining task-relevant information in each modality. Specifically, a modality should receive a lower regularization weight if a significant portion of its task-relevant information is left out of the MIB, and vice versa. However, existing MIB methods typically assign fixed, \textit{ad hoc} regularization weights to each modality during training. When task-relevant information is imbalanced across modalities, some modalities may contain minor but crucial task-relevant information (e.g., $a_2$ in $v_2$ in \cref{Fig: Information diagrams-1}). If such a modality is assigned an excessively large regularization weight, its task-relevant information may be inadvertently excluded from the MIB \cite{fan2023pmr}. Finally, these methods lack theoretical comprehensiveness, as they either fail to incorporate all five aforementioned factors into the optimization objective or do not acknowledge their distinct roles in guiding optimization. For instance, the study in \cite{tian2021farewell} overlooks complementarity, while CMIB-Nets \cite{wan2021multi} does not account for consistent, superfluous information. Additionally, in the theoretical analyses of methods such as \cite{Fang2024Dynamic, wan2021multi}, the two types of task relevant information—consistent (e.g., $a_0$ in \Cref{Fig: Information diagrams-1}) and modality-specific (e.g., $a_1, a_2$)— are not distinguished, despite their differing impacts on the optimization objective.  

To address these issues, we propose a novel MIB-based multimodal learning framework, termed \textbf{Optimal Multimodal Information Bottleneck (OMIB)}, to learn task-relevant optimal MIB representations from multi-modal data for enhanced downstream task performance. OMIB features theoretically grounded optimization objectives, explicitly linked to the dynamics of all five information-theoretical factors during optimization, ensuring a holistic and rigorous optimization framework.  
As illustrated in \Cref{Fig: framework}, OMIB comprises two components, including task relevance branches (TRBs) that extract sufficient representations from individual modalities, and an optimal multimodal fusion block (OMF), where modality-specific representations are fused by a cross-attention network (CAN) into MIB and optimized using a computationally efficient variational approximation \cite{alemi2017deep}. Adhering to the MIB principle, the OMF block maximizes sufficiency while minimizing redundancy in the MIB. Particularly, by setting the redundancy regularization parameter in OMIB's objective function within a theoretically derived bound, OMIB guarantees the achievability of optimal MIB upon convergence of the OMF block training. Furthermore, our approach dynamically refines regularization weights per modality \textit{as per} the distribution of their remaining task-relevant information. In summary, our contributions include: 
\begin{itemize}
\item We propose OMIB, a novel framework for learning optimal MIB representations from multimodal data, with an explicit solution to address imbalanced task-relevant information across modalities. 
\item We provide a rigorous theoretical foundation that underpins OMIB's optimization procedure, establishing a clear connection between its objectives and the five information-theoretical factors: sufficiency, consistency, redundancy, complementarity, and specificity. 
\item We mathematically derive the conditions for achieving optimal MIB, marking, to our knowledge, the first endeavor in proving its achievability under the MIB principle. 
\item We validate OMIB's effectiveness on synthetic data and demonstrate its superiority over state-of-the-art MIB methods in downstream tasks such as sentiment analysis, emotion recognition, and anomalous tissue detection across diverse real-world datasets.
\end{itemize}

\section{Related Work}

\subsection{Multimodal Fusion}

Multimodal fusion methods can be categorized according to the fusion stage and techniques. Architecturally, fusion can happen at three stages: (1) Early fusion, which combines data at the feature level \cite{snoek2005early}, (2) Late fusion, integrating data at the decision level \cite{morvant2014majority}, and (3) Middle fusion, where data is fused at intermediate layers to allow early layers to specialize in learning unimodal patterns \cite{nagrani2021attention}. From the technique perspective, fusion approaches include: (1) Operation-based, combining features through arithmetic operations \cite{el2020multimodal, lu2021ai}, (2) Attention-based, using cross-modal attention to learn interaction weights \cite{schulz2021multimodal, cai2023multimodal}, (3) Tensor-based, modeling high-order interactions \cite{chen2020pathomic, zadeh-etal-2017-tensor}, (4) Subspace-based, projecting modalities into shared latent spaces \cite{yao2017deep, zhou2021cohesive}, and (5) Graph-based, representing modalities as graph nodes and edges \cite{parisot2018disease, cao2021using}. 
In addition, recent studies also discuss the issue of modality imbalance, where strong modalities tend to dominate the learning process, while weak modalities are often suppressed \cite{peng2022balanced,zhang2024multimodal}. 
Though effective, these methods typically fail to account for superfluous information and thus are prone to overfitting and sensitive to noisy modalities, limiting their practical robustness \cite{Fang2024Dynamic}. MIB addresses these challenges by preserving task-relevant information while minimizing redundant content in the generated multimodal representations.

\subsection{Multimodal Information Bottleneck}
The Information Bottleneck (IB) framework \cite{tishby2000information} provides a principled approach for learning compressed, task-relevant representations. It was first applied to deep learning by \cite{tishby2015deep} and later extended through the Variational Information Bottleneck (VIB) \cite{alemi2017deep}, which employs stochastic variational inference for efficient approximations. Recently, IB has been adapted to more complex settings, such as multi-view \cite{wang2019deep, Federici2020Learning} and multimodal learning \cite{tian2021farewell}. For example, L-MIB, E-MIB, and C-MIB \cite{9767641} aim to learn effective multimodal representations by maximizing task-relevant mutual information, eliminating redundancy, and filtering noise, while exploring how MIB performs at different fusion stages. Secondly, MMIB-Zhang \cite{zhang2022information} improves multimodal representation learning by imposing mutual information constraints between modality pairs, enhancing the model's ability to retain relevant information. 
Additionally, DISENTANGLEDSSL \cite{wang2024an} relaxes the restrictions on achieving minimal sufficient information, thereby enabling the disentanglement of modality-shared and modality-specific information and enhancing interpretability. 
Lastly, DMIB \cite{Fang2024Dynamic} filters irrelevant information and noise, employing a sufficiency loss to preserve task-relevant data, ensuring robustness in noisy and redundant environments.\par

However, these methods often rely on \textit{ad hoc} regularization weights and overlook the imbalance of task-relevant information across modalities, limiting their ability to fully optimize the MIB framework.

\section{Notations}\label{sec:notation}
Here, we list the mathematical notations (\Cref{tab:notation}) used in this study.
\begin{table}[htbp]
    \centering
    \caption{Summary of notation.}
    \fontsize{7.8pt}{10pt}\selectfont
    \begin{tabular*}{\linewidth}{ll} 
\toprule
       \textbf{Notation}   & \textbf{Description} \\ \Xhline{0.6pt}
        $y$ & Task-relevant label.\\
        $v_i$&The $i$-th modality.\\ 
        $z_i$& The sufficient encoding of $v_i$ for $y$.\\
        $\xi$ & MIB encoding.\\ 
        $N$&The total number of observations.\\ 
        $H(*)$&The entropy of variable $*$.\\
        $F(*)$ &The information set inherent to variable * (i.e., $F(x)=H(x)$).\\
        $I$&The mutual information function.\\  
\bottomrule
    \end{tabular*}
    \label{tab:notation}
\end{table}

\section{Method}
\label{sec:method}
To clearly illustrate OMIB's framework, we start with the case of two data modalities (e.g., $v_1$ and $v_2$ in \Cref{Fig: framework}), which can be readily extended to multiple modalities by adding additional modality branches (see \Cref{multi method}). We also provide a rigorous theoretical foundation for our methodology in \Cref{sec:theoretical}. 

\paragraph{Warm-up training.} This phase consists of two task relevance branches (\textit{TRB}) corresponding to $v_1$ and $v_2$, respectively. In the $i$-th TRB, $v_i$ is first encoded into a sufficient representation $z_i \in \mathbb{R}^d$ for task-relevant labels $y$:
\begin{equation}
    z_i=Enc_i(v_i;\theta_{Enc_i}),
    s.t. I(z_i;y)=I(v_i;y),
\end{equation}
where $Enc_i$ is an encoder, $\theta_{Enc_i}$ denotes its parameters. To ensure maximal sufficiency of $z_i$ for $y$, we concatenate it with a stochastic Gaussian noise, $e_i\in \mathbb{R}^{k}=N(0,I)$, before feeding it to a task-relevant prediction head $Dec_i$ (see Appendix \ref{net:Implem}) to yield the predicted output $\hat{y}_i$:
\begin{equation}
    \hat{y}_i=Dec_i([z_i,e_i])
\end{equation}
Via this step, $Enc_i$ is optimized to extract maximal task-relevant information from $v_i$, as it requires a higher signal-to-noise ratio in $z_i$ to yield accurate prediction from its corrupted version. The loss function for updating $Enc_i$ and $Dec_i$ is:
\begin{equation}\label{loss TRB}
\begin{split}
    L_{TRB_i}&=E_{v_i}[-\log p(\hat{y}_i|z_i,e_i)]\\&=-\frac{1}{N} \sum_{n=1}^{N} \log p(\hat{y}_i^n|z_i^n,e_i^n).
\end{split}
\end{equation} 
Note that the implementation of $\log p(\hat{y}_i|z_i,e_i)$ is task-specific. For classification tasks, it is implemented as $CE(\hat{y}_i||y)$, where $CE$ is the cross-entropy function; for SVDD-based anomaly detection \cite{pmlr-v80-ruff18a}, it is $||\hat{y}_i-c||$, where $c$ is the unit sphere center of normal observations (see \Cref{net:Implem}); for regression tasks, it is $-||\hat{y}_i-y||$. The algorithmic workflow of the warm-up training is described in \Cref{pesudo code}.

\paragraph{Main Training.} After the warm-up training, the model proceeds to main training, which includes an optimal multimodal fusion (\textit{OMF}) block in addition to the TRBs. In the OMF, $z_i,\forall i \in \{1,2\},$ is used to generate the mean $\mu_i\in \mathbb{R}^{k}$ and variance $\Sigma_i\in \mathbb{R}^{k\times k}$ of a Gaussian distribution using a variational autoencoder ($VAE_i$):
\begin{equation}
    \mu_i,\Sigma_i=VAE_i(z_i,\theta_{VAE_i}),
\end{equation}
where $\theta_{VAE_i}$ represents the parameters of $VAE_i$. For efficient training and direct gradient backpropagation, the reparameterization trick \cite{kingma2013auto} is applied to generate $\zeta_i\in \mathbb{R}^k$:
\begin{equation}\label{eq:reparam}
    \zeta_i=\mu_i+\Sigma_i\times \epsilon_i,\ \text{where}\ \epsilon_i\sim N(0, I).
\end{equation}
$\zeta_1$ and $\zeta_2$ are fused using a Cross-Attention Network (CAN) \cite{vaswani2017attention}, whose architecture is detailed in  \Cref{net:Implem}:
\begin{equation}
    \xi=CAN(\zeta_1,\zeta_2,\theta_{CAN}),
\end{equation}
where $\xi$ is the post-fusion embedding, which is then passed to a task-relevant prediction head $\widehat{Dec}$  to generate the final prediction $\hat{y}$:
\begin{equation}
    \hat{y}=\widehat{Dec}(\xi,\theta_{\widehat{Dec}}).
\end{equation}
Meanwhile, $\xi$ replaces the stochastic noise $e_i$ in $v_i$'s TRB to fuse with $z_i$, yielding $\hat{y}_i$ for computing $L_{TRB_i}$ and updating $Enc_i$ and $Dec_i$:
\begin{equation}
    \hat{y}_i=Dec_i([z_i,\xi]).
\end{equation}
As established in \Cref{Lemma 1}, the loss function for updating the components in OMF (i.e., $VAE_i$, $CAN$, and $\widehat{Dec}$) to achieve optimal MIB, $\xi$, is given by:

\begin{equation}\label{equ:IB loss imp}
\begin{split}
    L_{OMF} &= \frac{1}{N} \sum_{n=1}^{N} \mathbb{E}_{\epsilon_1}\mathbb{E}_{\epsilon_2} \left[ -\log q(y^n|\xi^n) \right] \\
    &\quad + \beta \left( KL \left[ p(\zeta_1^n|z_1^n) || \mathcal{N}(0, I) \right] \right. \\
    &\quad \left. + r \, KL \left[ p(\zeta_2^n|z_2^n) || \mathcal{N}(0, I) \right] \right).
\end{split}
\end{equation}\par

where $\beta$ is a hyper-parameter constraining redundancy between $\zeta_i$ and $z_i$, and $r$ is a dynamically adjusted weight balancing the regularization of $v_2$ relative to $v_1$. 
The implementation of $-log\, q(y|\xi)$ is task-specific, as previously stated. $KL[p(\zeta_i|z_i)|| \mathcal{N}(0, I)]$ represents the KL-divergence between $\zeta_i$ and the standard normal distribution. As shown in \Cref{Lemma 2}, $r$ is explicitly computed during training as:
\begin{small}
\begin{equation} \label{equ:r}
    r= 1 - tanh\Bigl(\ln \frac{1}{N} \sum_{n=1}^{N} \mathbb{E}_{\epsilon_1}\mathbb{E}_{\epsilon_2}\Bigl[\frac{KL(p(\hat{y}_2^n|\xi^n,z_2^n)||p(\hat{y}^n|\xi^n))}{KL(p(\hat{y}_1^n|\xi^n,z_1^n)||p(\hat{y}^n|\xi^n))}\Bigl]\Bigl) 
\end{equation}
\end{small}\par
Furthermore, \Cref{prop:Ach OMIB} provides a theoretical upper bound for setting $\beta$, ensuring that our methodology achieves optimal MIB. 
The algorithmic workflow of the main training procedure is detailed in \Cref{pesudo code}.
\paragraph{Inference.} During inference, the TRBs are disabled, and the trained modality-specific encoder ($Enc_i$) and OMF generate optimal MIBs for test data to be used in downstream tasks.   

\section{Theoretical Foundation} \label{sec:theoretical}
Due to space constraints, we focus on the theoretical analysis of two data modalities in this section and defer the analysis of multiple data modalities ($\ge 3$) to \Cref{Multiple Theoretical Foundation}. 

\subsection{Optimal Information Bottleneck for Multimodal Data with Imbalanced Task-Relevant Information}\label{sec:OMIB}
As proposed in \cite{alemi2017deep,Federici2020Learning,9767641,wang2019deep}, the Information Bottleneck (IB) principle aims to optimize two key objectives:
 \begin{equation}
     \text{(1)}\ maximize\ I(y;z)\ \text{and}\ \text{(2)}\ minimize\ I(v;z)
 \end{equation}
where $y$ represents task-relevant labels, $v$ the input data, $z$ the IB encoding. The first objective maximizes $z$'s expressiveness for $y$, while the second objective enforces $z$'s conciseness.  These objectives can be formulated as: 
\begin{equation}
    \max_{z} I(y;z)\ s.t.\ I(v;z)\leq I_c,
\end{equation}
where $I_c$ is the information constraint that limits the amount of retained input information. Introducing a Langrange multiplier $\beta>0$, the objective function is reformulated as:
\begin{equation}
    \max_{z}I(y;z)-\beta I(v;z).
\end{equation}
For two data modalities, we propose a modified objective function to account for imbalanced task-relevant information across modalities:
\begin{equation} \label{equ:IB}
\min_{\xi} \ell(\xi) = \min_{\xi} -I(\xi;y)+\beta (I(\xi;v_1)+r I(\xi;v_2)),
\end{equation}
where $r>0$ is a dynamically adjusted parameter controlling the relative regularization of $v_2$ with respect to $v_1$.
In \Cref{equ:IB}, $v_i$ can be replaced with $z_i$. To see this point, let $\bar{v}_1$ denote the information in $v_1$ that is not encoded in $z_1$. Then, we have:
\begin{equation}
\begin{split}
I(\xi;v_1)&=I(\xi;z_1,\bar{v}_1)=I(\xi;z_1)+\underbrace{I(\xi;\bar{v}_1|z_1)}_{=0\ \because\ F(\xi)\ \cap F(\bar{v}_1)\ =\emptyset}\\ 
    &=I(\xi;z_1).
\end{split} 
\end{equation}
Similarly, $I(\xi;v_2)=I(\xi;z_2)$. Thus, the objective function in \Cref{equ:IB} can be rewritten as:
\begin{equation} \label{equ:IB alter}
\min_{\xi} \ell(\xi) = \min_{\xi} -I(\xi;y)+\beta (I(\xi;z_1)+r I(\xi;z_2)).
\end{equation}

\begin{proposition}[\textbf{Variational upper bound of OMIB's objective function}]
\label{Lemma 1}
 The loss function $ L_{OMF}$ in \Cref{equ:IB loss imp} provides a variational upper bound for optimizing the objective function in \Cref{equ:IB alter} and can be explicitly calculated during training.  
\end{proposition}
\begin{proof}
See Appendix \ref{AM-VIB}.
\end{proof}
Moreover, when a substantial portion of task-relevant information remains in $v_2$ relative to $v_1$, the value of $r$ should be small to encourage incorporating more information from $v_2$ in subsequent training iterations. Simultaneously, $r$ should be bounded to prevent over-regularizing information from $v_2$. Thus, $r$ can be mathematically expressed as: 
\begin{equation}\label{eq:ratio}
    \begin{split}
         r \propto \frac{I(y;v_1|\xi)}{I(y;v_2|\xi)}, r\in (0, u),
    \end{split} 
\end{equation}
where $I(y;v_i|\xi)$ represents the amount of task-relevant information in $v_i$ not encoded in $\xi$, and $u$ is an upper bound. In this study, $u$ is set to 2, as it is implemented using a $tahn$ function as in \Cref{equ:r}, which is justified by the following proposition.

\begin{proposition}[\textbf{Explicit formula for $r$}]
\label{Lemma 2}
\Cref{equ:r} satisfies \Cref{eq:ratio}, providing an explicit formula for computing $r$ during training.  
\end{proposition}
\begin{proof}
See Appendix \ref{AM-VIB}.
\end{proof}
In the next section, we establish a theoretical bound for $\beta$, ensuring that $\xi$ attains optimality during the optimization of the objective function in \Cref{equ:IB alter}.

\subsection{Achievability of Optimal Multimodal Information Bottleneck}\label{sec:Achieve OMIB}
\begin{assumption}\label{Assumption 1.1} 
As illustrated in \Cref{Fig: Information diagrams-1}, given two modalities, $v_1$ and $v_2$, the task-relevant information set $\{a\}$ consists of three components: $a_0, a_1$, and $a_2$. Specifically, $a_0$ is shared between both modalities, while $a_1$ and $a_2$ are specific to $v_1$ and $v_2$, respectively.  The task-relevant labels $y$ are determined by $\{a\}$. Moreover, $v_1$ and $v_2$ contain modality-specific superfluous information $b_{1}$ and $b_{2}$, respectively, in addition to shared superfluous information $b_0$.
\end{assumption} 

\begin{definition}[\textbf{Optimal multimodal information bottleneck}]
\label{def:optimal MIB}
 Under \cref{Assumption 1.1}, the optimal MIB, $\xi_{opt}$, for $v_1$ and $v_2$ satisfies:
 \begin{equation}
     F(\xi_{opt})=\{a_0,a_1,a_2\},
 \end{equation}
 ensuring that $\xi_{opt}$ encompasses all task-relevant information ($a_0,a_1$, and $a_2$) while exempting from superfluous information ($b_0, b_1$, and $b_2$).   
\end{definition}

\begin{figure*}[t]
\centering
\includegraphics[width=0.9\linewidth]{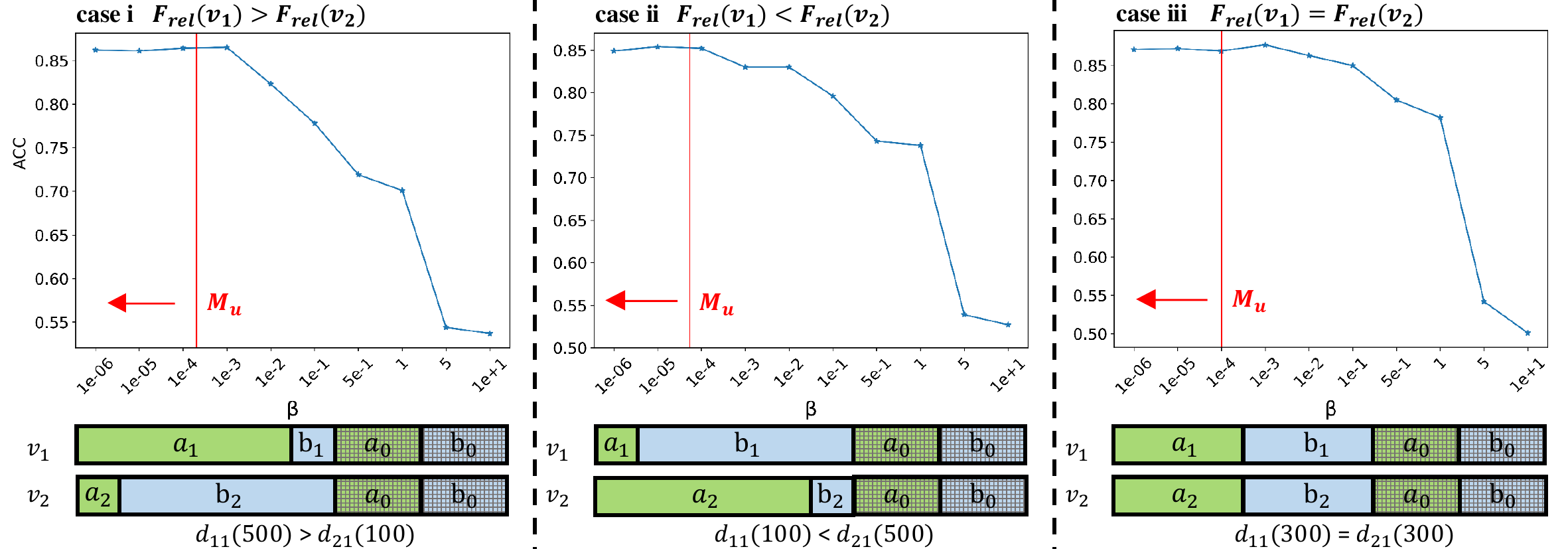} 
\caption{The impact of $\beta$ values on classification accuracy on synthetic data. $v_1$ and $v_2$ represent sample vectors of two modalities, respectively. 
$F_{\text{rel}}(\cdot)$ denotes task-relevant information. ``$a$" sub-vectors denote task-relevant information, while ``$b$" superfluous information. $d_{11}$ and $d_{21}$ denote the dimensions of  modality-specific $a_1$ and $a_2$. $M_u$ is the computed $\beta$ upper bound.
}.
\label{Fig: exp1}
\end{figure*}\par

\begin{lemma}[\textbf{Inclusiveness of task-relevant information}]
    \label{prop:inclusiveness}
Under \Cref{Assumption 1.1}, the objective function in \Cref{equ:IB alter} guarantees:
\begin{equation}
    F(\xi)\supseteq\{a_0,a_1,a_2\},
\end{equation}
provided that $\beta \in (0, M_u]$, where $M_u:=\frac{1}{(1+r) (H(v_1)+H(v_2)-I(v_1;v_2))}.$

\end{lemma}
\begin{proof}
    See Appendix \ref{appendix:achievability}
\end{proof}
Note that $H(v_1)+H(v_2)-I(v_1;v_2)$ represents the total information encompassed by the two data modalities. Intuitively, a larger total information content requires incorporating more information from each modality into the MIB. This is achieved by setting a lower $M_u$, ensuring that all task-relevant information is included in the MIB.

\begin{lemma}[\textbf{Exclusiveness of superfluous information}]
    \label{prop:exclusiveness}
    Under \Cref{Assumption 1.1}, the objective function in \Cref{equ:IB alter} ensures: \begin{equation}
F(\xi) \subseteq \{a_0,a_1,a_2\}\ 
\end{equation} 
\end{lemma}
\begin{proof}
    See \Cref{appendix:achievability}
\end{proof}
\begin{proposition}[\textbf{Achievability of optimal MIB}]
    \label{prop:Ach OMIB}
Under \Cref{Assumption 1.1}, the optimal MIB $\xi_{opt}$ is achievable through optimization of \Cref{equ:IB alter} with $\beta \in (0, M_u]$.
\end{proposition}
\begin{proof}
 \Cref{prop:inclusiveness} and \Cref{prop:exclusiveness} jointly demonstrate that $F(\xi)\supseteq\{a_0,a_1,a_2\}$ and $F(\xi) \subseteq \{a_0,a_1,a_2\}$, given $\beta \in (0, M_u]$. Thus, $F(\xi)=\{a_0,a_1,a_2\}$, which corresponds to $\xi_{opt}$ in \Cref{def:optimal MIB}. This completes the proof. 
\end{proof}
In this study, we set $M_u:=\frac{1}{3(H(v_1)+H(v_2)-I(v_1;v_2))}<\frac{1}{(1+r) (H(v_1)+H(v_2)-I(v_1;v_2))}$ as a tighter upper bound for $\beta$ given that $r\in (0,2)$, and $M_l:=\frac{1}{3(H(v_1)+H(v_2))}\leq M_u$ as a lower bound for $\beta$ to accelerate training. Importantly, both $M_l$ and $M_u$ can be computed a priori from the training data using the Mutual Information Neural Estimator (MINE, \cite{belghazi2018mine}) to estimate $H(\cdot)$ and $I(\cdot;\cdot)$ (see \Cref{net:estimate}).

\section{Experiment} 
Due to space constraints, we defer detailed task-specific experimental settings to  \Cref{exp_set} and implementations of network architectures to \Cref{net:Implem}. Detailed descriptions of the benchmark methods and evaluation metrics are provided in \Cref{benchmark} and \Cref{evaluation} respectively. The best and second-best performing methods in each experiment are bolded and underlined, respectively.

\begin{table}[htbp]
\centering
\caption{Classification accuracy of synthetic features vs. OMIB-generated MIB on simulated datasets.}
\fontsize{8pt}{9.5pt}\selectfont
\setlength{\tabcolsep}{4pt} 
\resizebox{1.0\linewidth}{!}{
\begin{tabular}{c|c|c}
\Xhline{1.0pt}
Datasets   &  Imbalanced (SIM-I)      &  balanced(SIM-III)     \\ \Xhline{0.7pt}
Consistent\& relevant   & 0.707  & 0.686    \\ 
Modality-specific\& relevant     & 0.737 & 0.744          \\ 
Unimodal      & 0.748 / 0.82   &  0.792 / 0.78        \\ 
Authentic optimal MIB    & \textbf{0.909} & \textbf{0.908}       \\ 
Union of two modalities & 0.858 & 0.866\\
\Xhline{0.7pt}
OMIB-generated MIB      & \underline{0.892}  & \underline{0.890}   \\ 
\Xhline{1.0pt}
\end{tabular}}
\label{Balanced vs. Imbalanced}
\end{table}

\subsection{Datasets}
To facilitate the analysis of OMIB's performance and validate \Cref{prop:Ach OMIB}, we simulate three Gaussian-based two-modality dataset, \textbf{SIM-\{I-III\}}, for classification (see \Cref{gen_sg}). Each dataset contains all four types of information ($\{\text{consistent},\text{modality-specific}\}\times\{\text{task-relevant},\text{superfluous}\}$). Moreover, they are synthesized with varying distributions of task-relevant information across modalities. \par

The emotion recognition experiment is conducted on \textbf{CREMA-D} \cite{cao2014crema}, an audio-visual dataset in which actors express six basic emotions—happy, sad, anger, fear, disgust, and neutral—through both facial expressions and speech. The MSA experiment utilizes \textbf{CMU-MOSI} \cite{zadeh2016multimodal}, which encompasses visual, acoustic, and textual modalities, with sentiment intensity annotated on a scale from -3 to 3. The pathological tissue detection experiment involves eight datasets derived from healthy human breast tissues (\textbf{10x-hNB-\{A-H\}}) and human breast cancer tissues (\textbf{10x-hBC-\{A-H\}})  \cite{xu2024detecting}, where each dataset comprises gene expression and histology modalities. OMIB is trained on the healthy datasets and applied to the cancer datasets for pathological tissue detection. Detailed descriptions of these datasets are provided in \Cref{dataset:description} and \Cref{table:dataset}. 

\begin{table*}[htbp]
\centering
\caption{Comparison of multimodal fusion methods for emotion recognition on the CREMA-D.}
\fontsize{7.5pt}{9.5pt}\selectfont
\setlength{\tabcolsep}{2pt} 
\begin{tabular}{c|ccc|cccccc|c}
\Xhline{0.9pt}
\multirow{2}{*}{Methods} &\multicolumn{3}{c|}{non-MIB-based} &\multicolumn{6}{c|}{MIB-based} &\multirow{2}{*}{OMIB} \\ \Xcline{2-10}{0.5pt}
 & Concat & BiGated & MISA & deep IB & MMIB-Cui & MMIB-Zhang & E-MIB & L-MIB &C-MIB  \\ \hline
Acc &53.2  & 58.4 &57.7  &54.1  &57.3  &56.7  & \underline{61.4} & 58.1  & 57.0  &\textbf{63.6} \\ \Xhline{0.9pt}
\end{tabular}
\label{class:CREMA-D}
\end{table*}

\begin{table}[thbp]
\centering
\caption{Comparison of multimodal fusion methods for sentiment analysis on the CMU-MOSI dataset.}
\fontsize{6.5pt}{7.8pt}\selectfont
\setlength{\tabcolsep}{3pt}
\resizebox{1.0\linewidth}{!}{
\begin{tabular}{c|c|c|c|c|c}
\Xhline{0.7pt}
Method       & Acc7  ($\uparrow$) & Acc2 ($\uparrow$) & F1($\uparrow$) & MAE($\downarrow$) & Corr($\uparrow$) \\ \hline
Concat &41.5 &81.1 &82.0 &0.797 &0.745 \\
BiGated  &41.8 &82.1 &83.2 &0.787 &0.738 \\
MISA & 42.3 & 83.4 & 83.6 & 0.783 &0.761\\
\hline
deep IB &45.3 &83.2 &83.3 &0.747 &0.785 \\
MMIB-Cui &45.7 &84.3 	&84.4 	&0.726	&0.782 \\
MMIB-Zhang &46.3  &85.0 &85.0 &0.713 &0.788  \\ 
DMIB &40.4 &83.2 &83.3 &0.810 &0.784 \\
E-MIB &\textbf{48.6} &\underline{85.3} &\underline{85.3} &\underline{0.711} &\underline{0.798} \\
L-MIB &45.8 &84.6 &84.6 &0.732 &0.790 \\
C-MIB &\underline{48.2} &85.2 &85.2 &0.728 &0.793\\
\hline
OMIB &\textbf{48.6} &\textbf{86.9} &\textbf{87.1} &\textbf{0.709} &\textbf{0.802} \\ \Xhline{0.7pt}
\end{tabular}}
\label{MOSI}
\end{table} 

\subsection{Empirical Analysis of OMIB Performance Using Synthetic Data} \label{exp:Synthetic Gaussian}

To empirically validate the effectiveness of our proposed $\beta$'s upper bound in achieving optimal MIB, we simulate three two-modality datasets (\textbf{SIM-\{I-III\}}) corresponding to three experimental cases (\textbf{case i-iii}) (see \Cref{gen_sg}). Regarding task-relevant information, Modality I dominates Modality II in SIM-I, Modality II dominates Modality I in SIM-II, and both modalities contribute equally in SIM-III, thereby covering the three primary cross-modal task-relevant information distributions observed in practice. Each dataset is designed for a binary classification task with labels $y\in\{0,1\}$. In each experimental case, $\beta$ is gradually increased from $10^{-6}$ to 10, well exceeding the proposed upper bound $M_u$. The generated MIBs are fed into the trained OMF prediction head to predict $y$ during testing. As shown in \Cref{Fig: exp1}, the prediction accuracy consistently peaks across all cases when using MIBs generated with $\beta$ near or below $M_u$, but rapidly declines as $\beta$ further increases. This observation aligns with our theoretical analysis, empirically confirming that optimal MIB is achievable when $\beta\leq M_u$. Notably, since $M_u$ is a tight upper bound, peak performance may still be observed for $\beta$ values slightly above $M_u$. 

As detailed in \Cref{gen_sg}, let $x_1=[a_0;b_0;a_1;b_1]$ and $x_2=[a_0;b_0;a_2;b_2]$ denote feature vectors of two observations in Modality I and II, respectively. Here, $a_0$ and $b_0$ correspond to the task-relevant and superfluous sub-vectors shared by both modalities. $a_1,a_2$ are modality-specific, task-relevant sub-vectors, while $b_1, b_2$ are modality-specific, superfluous sub-vectors. By design, the authentic optimal MIB is $[a_0;a_1;a_2]$, which is used to predict $y$ and compared against the prediction using OMIB-generated MIB. Additionally, we evaluate prediction accuracy using other feature sub-vectors, including unimodal information ($x_1$ or $x_2$), consistent task-relevant information ($[a_0]$), modal-specific task-relevant information ($[a_1;a_2]$), and complete information ($[a_0;b_0;a_1;b_1;a_2;b_2]$). This experiment is conducted using SIM-I and SIM-II, corresponding to the cases of imbalanced and balanced task-relevant information, respectively. Table \ref{Balanced vs. Imbalanced} demonstrates that OMIB-generated MIB achieves prediction accuracy most comparable to the authentic optimal MIB, surpassing all other feature sub-vector configurations that either omit task-relevant information or include superfluous information. These results further validate the optimality of OMIB-generated MIB.

\begin{table*}[htbp]
\centering
\caption{Comparison of  multimodal fusion methods for anomalous tissue detection performance on the 10x-hBC-\{A-D\} datasets}
\fontsize{6.5pt}{7.5pt}\selectfont
\setlength{\tabcolsep}{2.7pt} 
\renewcommand{\arraystretch}{1}
\resizebox{1.0\linewidth}{!}{
\begin{tabular}{c|c|ccc|ccccccc|c}
\Xhline{0.7pt}
\multirow{2}{*}{\centering \makecell{Target \\ Dataset}}& \multirow{2}{*}{\centering Metric}& \multicolumn{3}{c|}{non-MIB-based} & \multicolumn{7}{c|}{MIB-based} &\multirow{2}{*}{OMIB}\\

\Xcline{3-12}{0.4pt}
& &Concat &BiGated &MISA &deep IB &MMIB-Cui &MMIB-Zhang &DMIB & E-MIB & L-MIB  & C-MIB & \\ 

\Xhline{0.5pt}
\multirow{2}{*}{10x-hBC-A} & AUC &0.537 &0.489 &0.498 &0.522 &0.623 &\underline{0.626} &0.423 & 0.511
&0.598 &0.496 & \textbf{0.728} \\
& F1 &0.884 &0.821 &0.873 &0.878 &0.894 &\underline{0.897} &0.865 & 0.877 &0.891 &0.881 & \textbf{0.904} \\
 
\Xhline{0.5pt}
\multirow{2}{*}{10x-hBC-B} & AUC &0.866 &0.518 &0.499 &0.379 &0.818 &0.817 &\underline{0.849} &0.643 &0.770 &0.481 &\textbf{0.903} \\
& F1   &0.654 &0.352 &0.213 &0.102 &0.559 &0.583 &\underline{0.607} &0.330 &0.483 &0.213 & \textbf{0.663} \\

\Xhline{0.5pt}
\multirow{2}{*}{10x-hBC-C} & AUC &0.638 &0.563 &0.586 &0.433 &\textbf{0.765} &0.662 &\underline{0.743} &0.598 &0.659 &0.511 &\underline{0.743}\\
& F1 &0.750 &0.727 &0.754 &0.693 &	\underline{0.822} &0.783 &\textbf{0.827} &0.759 &0.786 &0.723 &0.820 \\
 
\Xhline{0.5pt}
\multirow{2}{*}{10x-hBC-D} &AUC &0.555
 &0.540 &0.495 &0.484 &0.501 &0.604 &\underline{0.642} &0.530 &\textbf{0.652} &0.503 &0.640 \\
& F1 &0.509 &0.494 &0.450 &0.443 &0.465 &0.524 &0.540 &0.483 &\textbf{0.564} &0.465 &\underline{0.561} \\

\Xhline{0.5pt}
\multirow{2}{*}{Mean} & AUC &0.649 &0.528 &0.520 &0.455 &\underline{0.677} &\underline{0.677} &0.664 &0.571 &0.602 &0.498 &\textbf{0.754} \\
& F1 &\underline{0.699} &0.599 &0.573 &0.529 &0.685 &0.697 &0.710 &0.612 &0.681 &0.571 &\textbf{0.737} \\

\Xhline{0.8pt}
\end{tabular}}
\label{Anomaly detection}
\end{table*} 

\subsection{Emotion Recognition}\label{sec:ER}
Here, we compare the accuracy of classifying actors' emotion types in the CREMA-D dataset using OMIB and ten benchmark methods, including three non-MIB-based fusion methods (concatenation, FiLM \cite{perez2018film}, and BiGated \cite{kiela2018efficient}) and seven MIB-based state-of-the-art (SOTA) methods (E-MIB, L-MIB, and C-MIB \cite{9767641} ). The classification accuracy of each method is reported in Table \ref{class:CREMA-D}. OMIB outperforms all other methods, achieving improvements of 8.9\% and 3.6\% over the best-performing non-MIB-based (concatenation) and MIB-based (E-MIB) fusion methods, respectively. These results underscore OMIB's superiority in enhancing emotion recognition performance. 

\subsection{Multimodal Sentiment Analysis}
To evaluate OMIB's effectiveness in improving downstream tasks involving three modalities, we conduct MSA on the CMU-MOSI dataset, which includes visual, acoustic, and textual modalities. Specifically, OMIB and the same benchmark methods from \Cref{sec:ER} are used to predict a real-valued sentiment intensity score for each utterance, ranging from -3 to 3. 
Evaluation metrics for this experiment are mentioned in \Cref{evaluation}. Additionally, OMIB consistently outperforms all benchmark methods across all evaluation metrics, highlighting its ability to generate high-quality MIB in a three-modal setting for enhanced regression tasks such as the MSA.

\subsection{Anomalous Tissue Detection}
In this experiment, we aim to identify anomalous tissue regions from the eight human breast cancer datasets (10x-hBC-\{A-D\}), which include gene expression and histology modalities. Due to the scarcity of tissue region annotations, we adopt the SVDD strategy \cite{pmlr-v80-ruff18a} for anomaly detection. Specifically, the model is trained exclusively on the eight healthy datasets (10x-hNB-\{A-H\}) to learn a compact hypersphere in the latent space, confining the multimodal representations of inliers. The trained model is then applied to the four breast cancer target datasets, generating multimodal representations whose distances to the center of the hypersphere serve as anomaly scores, based on which anomalous regions are identified. The benchmark methods are the same as those in \Cref{sec:ER} and modified to accommodate the SVDD strategy. The implementation details of OMIB for this task, are provided in \Cref{net:Implem}. 
The detection results are evaluated using the AUC and F1 scores, calculated based on the anomalous scores (see \Cref{evaluation}).  \Cref{Anomaly detection} demonstrates that OMIB consistently surpasses the best-performing benchmark method by an average leap of 11.4\% in AUC and 3.8\% in F1-score across the target datasets, confirming its superiority in anomaly detection in a multi-modal setting. 

\begin{table}[htbp]
\centering
\caption{Ablation studies on the CREMA-D dataset.}
\fontsize{7pt}{8.7pt}\selectfont
\setlength{\tabcolsep}{2.4pt} 
\resizebox{1.0\linewidth}{!}{
\begin{tabular}{c|cccc|c}
\Xhline{0.7pt}
 & w/o Warm-up  & w/o cross-attn & w/o OMF & w/o $r$ & Full \\ \hline
Acc   & 60.3 & 61.5  & 59.5 &62.2 & \textbf{63.6} \\ 
\Xhline{0.7pt}
\end{tabular}}
\label{exp:ablation}
\end{table}

\subsection{Ablation Study}
To gain deeper insight into the key components of OMIB, we conduct a series of ablation experiments on the CREMA-D dataset (\Cref{exp:ablation}). First, we examine the effect of removing the warm-up training ('w/o warm-up'), which leads to a 5.5\% decline in accuracy. Next, we replace the CAN with simple concatenation fusion ('w/o cross-attn'), resulting in a 3.4\% drop in accuracy. We also evaluate the effect of replacing the entire OMF block with simple concatenation fusion ('w/o OMF'), which significantly degrades model performance by 6.9\% in accuracy. Finally, we assign equal regularization weights to $I(\xi;z_1)$ and $I(\xi;z_2)$ by omitting $r$ ('w/o $r$') and observe a performance decline of 2.3\% in accuracy. 
In a nutshell, the degraded performance observed after removing OMIB's key components highlights their critical roles in ensuring model performance. 

\subsection{Complexity Analysis}\label{complexity}

\begin{figure}[h] 
\centering
\includegraphics[width=0.95\linewidth]{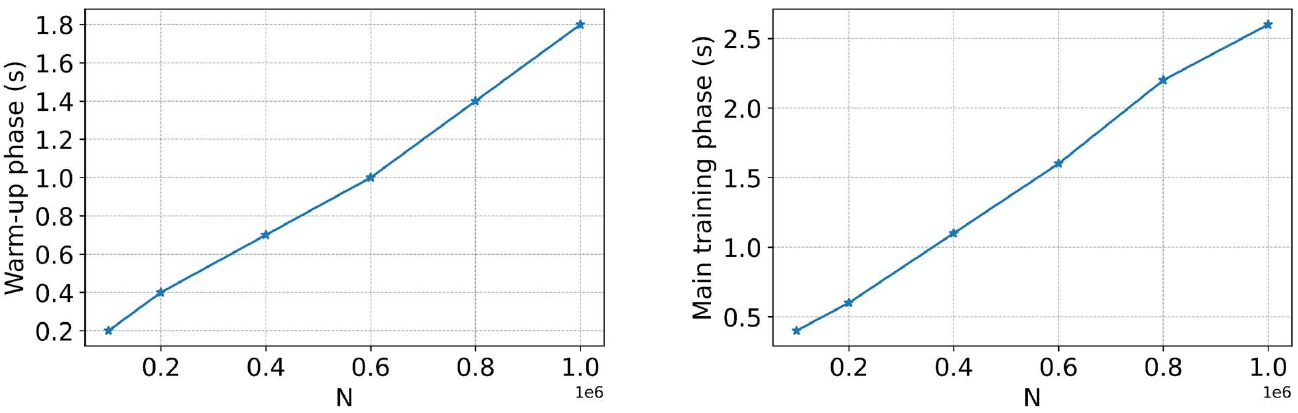} 
\caption{Runtime per epoch during warm-up and main training phase on synthetic data.}
\label{Fig: complex}
\end{figure}\par

We first provide a theoretical analysis of OMIB's complexity. OMIB's modality-specific encoder ($Enc$), task-relevant prediction head ($Dec$ and $\widehat{Dec}$), and VAEs are implemented as Multilayer Perceptron (MLP), convolutional network, or graph convolutional network, each with a complexity of $\mathcal{O}(N)$, where $N$ denotes the number of samples \cite{he2015convolutional, lecun2002efficient, wu2020comprehensive}. For the CAN network, our implementation (see  \Cref{net:Implem}) has a time complexity of $\mathcal{O}(N \cdot M^2)$ \cite{vaswani2017attention}, where $M$ represents the number of modalities. Since $M$ is typically small, $M^2$ can be treated as a constant. Thus, OMIB's overall theoretical complexity is $\mathcal{O}(N)$. We also empirically evaluate OMIB's scalability to input size using the SIM-III dataset. Explicitly, we sample six datasets with sizes: \(1 \times 10^5\), \(2 \times 10^5\), \(4 \times 10^5\), \(6 \times 10^5\), \(8 \times 10^5\), and \(1 \times 10^6\), while keeping the experimental settings identical to those of \textbf{case iii} in Section 6.2. We conduct separate analyses for the warm-up and main training phases, both of which demonstrate scalability to input size, as shown in \Cref{Fig: complex}.

\section{Conclusion}
We have proposed the OMIB framework, designed to learn optimal MIB representations that effectively capture all task-relevant information. Through theoretical analysis, we demonstrate that adjusting the weights of the IB loss across different modalities facilitates the achievement of optimal MIB. Our experimental results show that OMIB outperforms existing MIB-based methods. Furthermore, our approach is robust, successfully achieving optimal MIB regardless of whether the SNRs between modalities are balanced or imbalanced.\par

\section*{Impact Statement}
This paper presents work whose goal is to advance the field of Machine Learning. There are many potential societal consequences of our work, none which we feel must be specifically highlighted here.

\section*{Acknowledgments}
We would like to thank Wenlin Li, Yan Lu, Zhengke Duan, and Junqi Li for their help with the experiments.

\nocite{langley00}

\bibliography{ref}
\bibliographystyle{icml2025}

\newpage
\appendix
 \onecolumn

\section*{\LARGE Appendix}
\section{Proofs of Mutual Information Properties}\label{proof:prop}

\begin{properties}\label{properties 1}
Properties of mutual information and entropy:
\begin{flalign*}
\textbf{i}) \ &I(x;y)\geq 0, I(x;y|z)\geq 0. &
\end{flalign*}
\begin{flalign*}
\textbf{ii}) \ &I(x;y,z) = I(x;y)+I(x;z|y). &
\end{flalign*}
\begin{flalign*}
\textbf{iii}) \ I(x_1;x_2;\cdots;x_{n+1}) &= I(x_1;\cdots;x_n)\\ &-I(x_1;\cdots;x_n|x_{n+1}). &
\end{flalign*}
\begin{flalign*}
\textbf{iv}) \ \text{If} \ &F(x_1)\cap F(x_2)=\emptyset \longrightarrow I(x_1;x_3|x_2)=I(x_1;x_3)&
\end{flalign*}
\begin{flalign*} 
\textbf{v})\ \text{If} \ &F(v_2)\subseteq F(v_1)\longrightarrow I(v_1;v_2)=H(v_2), \\ &H(v_1,v_2)= F(v_1)\cup F(v_2)=F(v_1)=H(v_1) &
\end{flalign*}
\begin{flalign*}
\textbf{vi}) \ \text{If}\ &H(v_2)\cap H(v_1)\neq \emptyset \longrightarrow H(v_1,v_2)=H(v_1)+H(v_2)\\&-I(v_1;v_2) &
\end{flalign*}
\begin{flalign*}
\textbf{vii}) \ \text{If}\ &H(v_2)\cap H(v_1)=\emptyset \longrightarrow H(v_1,v_2)=H(v_1)+H(v_2)\\&=F(v_1)\cup F(v_2) &
\end{flalign*}
\end{properties}

\begin{proof}
The proofs of properties \textbf{i, ii,} and \textbf{iii} can be found in \cite{cover1999elements}. For property \textbf{iv}, we first observe that:
\begin{equation}
    \begin{split}
        F(y)\cap F(z)=\emptyset \longrightarrow p(y,z) = p(y)p(z)
    \end{split}
\end{equation}\par
This implies that $y$ and $z$ are statistically independent. Consequently, we have
\begin{equation}
    \begin{split}
        I(y;z) &= \sum_{y,z} p(y,z) log\, \frac{p(y,z)}{p(y)p(z)}\\
        &= \sum_{y,z} p(y,z) log\, \frac{p(y)p(z)}{p(y)p(z)}\\
        &= \sum_{y,z} p(y,z) log\, 1\\
        &=0
    \end{split}
\end{equation}\par
Given that $I(y;z)=I(x;y;z)+I(x;z|y)$, and noting that $I(x;y;z)\geq 0$ and $ I(x;z|y)\geq 0$, it follows that:
\begin{equation}
    \begin{split}
        I(y;z)=0 \longrightarrow I(x;y;z) = 0\, \text{and}\, I(x;z|y) = 0
    \end{split}
\end{equation}\par
Therefore, we obtain that:
\begin{equation}
    \begin{split}
        I(x;y|z) = I(x;y) - \overbrace{I(x;y;z)}^{=0} = I(x;y)
    \end{split}
\end{equation}

For property \textbf{v}:
    \begin{equation}
        \begin{split}
            H(v_1;v_2)&=\underset{v_1,v_2}{\iint }p(v_1,v_2)log(\frac{p(v_1,v_2)}{p(v_1)p(v_2)})\\ &=\underset{v_1,v_2}\iint p(v_1,v_2)log (\frac{\overbrace{p(v_2|v_1)}^{=1\ \text{as}\ F(v_2)\subseteq F(v_1)}p(v_1)}{p(v_1)p(v_2)})\\
&=\underset{v_2}{\int }-p(v_2)log(p(v_2))=H(v_2).
        \end{split}
    \end{equation}\par
In addition, for $I(v_1,v_2)$, we have:
\begin{equation}
    \begin{split}
            H(v_1,v_2)&=F(v_1)\cup F(v_2)\\&=\underset{v_1,v_2}{\iint }-p(v_1,v_2)log( p(v_1,v_2))\\ 
            &=\underset{v_1,v_2}\iint -p(v_1,v_2)log(p(v_2|v_1)p(v_1))\\
            &=\underset{v_1}\int -p(v_1)log(p(v_1))=H(v_1)=F(v_1).   
    \end{split}
\end{equation}\par

For property \textbf{vi}, we have:
\begin{equation}
    \begin{split}
    H(v_1) \cap H(v_2) \neq \emptyset \longrightarrow p(v_1, v_2) \neq p(v_1) p(v_2). \\
\end{split}
\end{equation}\par
The mutual information $I(v_1; v_2)$ is defined as:
\begin{equation}
    \begin{split}
I(v_1; v_2) &= \iint p(v_1, v_2) \log \frac{p(v_1, v_2)}{p(v_1) p(v_2)} \, dv_1 \, dv_2 \neq 0 \\
&= \iint p(v_1, v_2) \log p(v_1, v_2) \, dv_1 \, dv_2 - \iint p(v_1, v_2) \log p(v_1) \, dv_1 \, dv_2 \\
&\quad - \iint p(v_1, v_2) \log p(v_2) \, dv_1 \, dv_2 \\
&= -\iint p(v_1, v_2) \log p(v_1) \, dv_1 \, dv_2 - \iint p(v_1, v_2) \log p(v_2) \, dv_1 \, dv_2 \\
&\quad + \iint p(v_1, v_2) \log p(v_1, v_2) \, dv_1 \, dv_2 \\
&= -\int p(v_1) \log p(v_1) \, dv_1 - \int p(v_2) \log p(v_2) \, dv_2 \\
&\quad + \iint p(v_1, v_2) \log p(v_1, v_2) \, dv_1 \, dv_2 \\
&= H(v_1) + H(v_2) - H(v_1, v_2) \\
\end{split}
\end{equation}\par
Since $H(v_1) \cap H(v_2) \neq \emptyset \longrightarrow I(v_1; v_2) \neq 0$, we have $H(v_1, v_2) = H(v_1) + H(v_2) - I(v_1; v_2)$.\par
For property \textbf{vii}, we first clarified that:
    \begin{equation}
         F(v_2)\cap F(v_1)=\emptyset  \longrightarrow p(v_1,v_2)=p(v_1)p(v_2)
    \end{equation}\par
Therefore, we have:
    \begin{equation}
        \begin{split}
            H(v_1,v_2)&=\underset{v_1,v_2}{\iint }-p(v_1,v_2)log (p(v_1,v_2))\\ 
            &=\underset{v_1,v_2}{\iint }-p(v_1)p(v_2)log (p(v_1)p(v_2))\\
            &=\underset{v_1}{\int }-p(v_1)log (p(v_1))+\underset{v_2}{\int }-p(v_2)log(p(v_2))\\
            &=H(v_1)+H(v_2)
        \end{split}
    \end{equation}
\end{proof}

\section{Proofs of \Cref{Lemma 1} and \Cref{Lemma 2}}\label{AM-VIB}
For convenient reading, the equations used in the proofs are copied from the main text:
\begin{equation}\label{equ:IB loss imp-2}
\begin{split}
    L_{OMF} &= \frac{1}{N} \sum_{n=1}^{N} \mathbb{E}_{\epsilon_1}\mathbb{E}_{\epsilon_2} \left[ -\log q(y^n|\xi^n) \right]
    + \beta \left( KL \left[ p(\zeta_1^n|z_1^n) || \mathcal{N}(0, I) \right]
    + r \, KL \left[ p(\zeta_2^n|z_2^n) || \mathcal{N}(0, I) \right] \right).
\end{split}
\end{equation}
which is copied from \ \Cref{equ:IB loss imp}. 
\begin{equation} \label{equ:r-2}
    r= 1 - tanh\Bigl(\ln \frac{1}{N} \sum_{n=1}^{N} \mathbb{E}_{\epsilon_1}\mathbb{E}_{\epsilon_2}\Bigl[\frac{KL(p(\hat{y}_2^n|\xi^n,z_2^n)||p(\hat{y}^n|\xi^n))}{KL(p(\hat{y}_1^n|\xi^n,z_1^n)||p(\hat{y}^n|\xi^n))}\Bigl]\Bigl),
\end{equation}
which is copied from \ \Cref{equ:r}. 
\begin{equation} \label{equ:IB alter-2}
\min_{\xi} \ell(\xi) = \min_{\xi} -I(\xi;y)+\beta (I(\xi;z_1)+r I(\xi;z_2)),
\end{equation}
which is copied from \ \Cref{equ:IB alter}
\begin{equation} \label{eq:ratio-2}
    \begin{split}
        r&\propto \frac{I(y;v_1|\xi)}{I(y;v_2|\xi)},
    \end{split} 
\end{equation}
which is copied from \ \Cref{eq:ratio}. 

\begin{proposition}[\textbf{Proposition \ref{Lemma 1} restated}]
\label{Lemma 1-repeat}
The loss function, $L_{OMF}$, in \Cref{equ:IB loss imp-2} provides a variational upper bound for optimizing the objective function in \Cref{equ:IB alter-2}, and can be explicitly calculated during training.
\end{proposition}
\begin{proof}
For $I(\xi;y)$, we have:
\begin{equation}
\begin{split}
I(\xi;y) &= \int dy d\xi p(y,\xi) log\, \frac{p(y,\xi)}{p(y)p(\xi)}\\
&= \int dy d\xi p(y,\xi) log\, \frac{p(y|\xi)}{p(y)}
\end{split}
\end{equation}
Let $q(y|\xi)$ be a variational approximation to $p(y|\xi)$, and we have:
\begin{equation}\label{eq:KL}
    KL[p(y|\xi)||q(y|\xi)]\geq 0 \Rightarrow \int dyp(y|\xi)\log p(y|\xi)\geq \int dyp(y|\xi)\log q(y|\xi)
\end{equation}
Based on the above inequality, we have \cite{alemi2017deep}: 
\begin{equation}
\begin{split}
I(\xi;y) &\geq \int dy d\xi p(y,\xi) log\, \frac{q(y|\xi)}{p(y)}\\
&= \int dy d\xi p(y,\xi) log\, q(y|\xi) - \int dy d\xi p(y,\xi) log\, p(y)\\
&= \int dy d\xi p(y,\xi) log\, q(y|\xi) + H(Y)\\
\end{split}
\end{equation}
$H(Y)$ can be ignored as it is fixed during training.  
Therefore:
\begin{equation}\label{ineq:1}
\begin{split}
I(\xi;y) &\geq \int dy d\xi p(y,\xi) log\, q(y|\xi)\\
&=\int dy d\xi d\zeta_1 d\zeta_2 dz_1dz_2p(z_1,z_2,y,\zeta_1,\zeta_2,\xi) log\, q(y|\xi)\\
\end{split}
\end{equation}
Furthermore, because $\xi$ is a function of $\zeta_1$ and $\zeta_2$ (i.e., $\xi=CAN(\zeta_1,\zeta_2)$), we have $I(\xi;z_1) \leq I(\zeta_1,\zeta_2;z_1)$ and $I(\xi;z_2) \leq I(\zeta_1,\zeta_2;z_2)$. Using the Markov property, we have $\zeta_1\perp z_2$ and $\zeta_2 \perp z_1$, which leads to:
\begin{equation}
\begin{split}
I(\xi;z_1) \leq I(\zeta_1,\zeta_2;z_1) = I(\zeta_1;z_1) + \overbrace{I(\zeta_2;z_1|\zeta_1)}^{\zeta_2\perp z_1} = I(\zeta_1;z_1)
\end{split}
\end{equation}
Similarly, $I(\xi;z_2)\leq I(\zeta_2;z_2)$. Therefore:  
\begin{equation}
\begin{split}
I(\xi;z_i) \leq I(\zeta_i;z_i) &=\int d\zeta_i dz_i p(\zeta_i,z_i) log\, \frac{p(\zeta_i|z_i)}{p(\zeta_i)}, \forall i\in \{1,2\}
\end{split}
\end{equation}
Let $r(\zeta_i)\sim \mathcal{N}(0, I)$ be a variational approximation to $p(\zeta_i)$, we have:
\begin{equation}\label{ineq:2}
\begin{split}
I(\xi;z_i) & \leq I(\zeta_i;z_i) = \int d\zeta_i dz_i p(\zeta_i,z_i) log\, p(\zeta_i|z_i)-\int p(\zeta_i)\log p(\zeta_i)d\zeta_i \\
&\leq \int d\zeta_i dz_i p(\zeta_i,z_i) log\, p(\zeta_i|z_i)-\int p(\zeta_i)\log r(\zeta_i)d\zeta_i\\
&= \int d\zeta_i dz_i p(\zeta_i,z_i) log\, \frac{p(\zeta_i|z_i)}{\mathcal{N}(0, I)}, \forall i \in \{1,2\}.
\end{split}
\end{equation}
Put \Cref{ineq:1} and \Cref{ineq:2} together, we have:
\begin{equation}
\begin{split}
L=&-I(\xi;y)+\beta \Bigl(I(\xi;z_1)+r I(\xi;z_2)\Bigl) \\ 
&\leq -\int dydz_1 dz_2p(y,z_1,z_2) \int d\xi d\zeta_1 d\zeta_2 p(\xi|\zeta_1,\zeta_2)p(\zeta_1|z_1)p(\zeta_2|z_2) log\, q(y|\xi) \\&+ \beta \Bigl(\int dz_1 p(z_1) \int d\zeta_1  p(\zeta_1|z_1) log\, \frac{p(\zeta_1|z_1)}{\mathcal{N}(0, I)} + r \int dz_2 p(z_2)\int d\zeta_2  p(\zeta_2|z_2) log\, \frac{p(\zeta_2|z_2)}{\mathcal{N}(0, I)}\Bigl) 
\end{split}
\end{equation}
Note that $p(z_1, z_2, y), p(z_1)$, and $p(z_2)$ can be approximated using the empirical data distribution \cite{alemi2017deep, wang2019deep},
which leads to the objective function:
\begin{equation}
\begin{split}
L &\approx \frac{1}{N} \sum_{n=1}^{N} \Bigl[ - \int d\xi d\zeta_1 d\zeta_2 \, p(\xi^n|\zeta_1^n,\zeta_2^n)p(\zeta_1^n|z_1^n)p(\zeta_2^n|z_2^n) log \, q(y^n|\xi^n) \\&+ \beta\Bigl(\int d\zeta_1 p(\zeta_1^n|z_1^n)log\, \frac{p(\zeta_1^n|z_1^n)}{\mathcal{N}(0, I)} + r \int d\zeta_2 p(\zeta_2^n|z_2^n)log\, \frac{p(\zeta_2^n|z_2^n)}{\mathcal{N}(0, I)}\Bigl) \Bigl]
\end{split}
\end{equation}
Given $\zeta_i=\mu_i+\Sigma_i\times \epsilon_i$ in \Cref{eq:reparam}, we have:
\begin{equation}
\begin{split}
\mathcal{L}  
&=\frac{1}{N} \sum_{n=1}^{N} \mathbb{E}_{\epsilon_1}\mathbb{E}_{\epsilon_2} \left[ -\log q(y^n|\xi^n) \right]
    + \beta \Bigl( KL \left[ p(\zeta_1^n|z_1^n) || \mathcal{N}(0, I) \right]
    + r \, KL \left[ p(\zeta_2^n|z_2^n) || \mathcal{N}(0, I) \right] \Bigl)\\
&=L_{OMF}
\end{split}
\end{equation}\par
This completes the proof.
\end{proof}

\begin{proposition}[\textbf{Proposition \ref{Lemma 2} restated}]\label{Lemma 2-repeat}
\Cref{equ:r-2} satisfies \Cref{eq:ratio-2}, thus providing an explicit formula for computing $r$ during training.  
\end{proposition}
\begin{proof} 
Firstly, $z_1$ and $z_2$ are sufficient encodings of modalities $v_1$ and $v_2$ for $y$, respectively. 
Let $\bar{v}_i$ represent the superfluous information in $v_i$ that is not encoded in $z_i$. Then, we have:
\begin{equation}
\begin{split}
    I(y;v_i|\xi) &= I(y;z_i,\bar{v}_i|\xi) \\ 
    &= I(y;z_i|\xi) + \underbrace{I(y;\bar{v}_i|z_i,\xi)}_{=0\ \because\ F(y) \cap F(\bar{v}_i) = \emptyset}\\
    &= I(y;z_i|\xi), \forall i\in \{1,2\}.
\end{split} 
\end{equation}
\par
Let $z_4 = \{z_1,\xi\}$ and $z_5 = \{z_2,\xi\}$, then we have:
\begin{equation}
\begin{split}
    I(y;v_1|\xi) &= I(y;z_1|\xi) = I(z_1,\xi;y) - I(\xi;y) = I(z_4;y) - I(\xi;y),\\
    I(y;v_2|\xi) &= I(y;z_2|\xi) = I(z_2,\xi;y) - I(\xi;y) = I(z_5;y) - I(\xi;y).
\end{split}
\end{equation}
Then $I(y;z_1|\xi)$ can be expressed as:
\begin{equation}
\begin{split}
I(y;z_1|\xi) &=I(z_4;y)-I(\xi;y) \\&= H(y)-H(y|z_4) - H(y) + H(y|\xi)\\ 
&=H(y|\xi) - H(y|z_4)\\
&=-\int p(\xi)d\xi \int p(y|\xi) log\, p(y|\xi)dy+\int p(z_4)dz_4 \int p(y|z_4) log\, p(y|z_4)dy\\ 
&=-\iint p(\xi)p(y|\xi) log\, [p(y|z_4)\frac{p(y|\xi)}{p(y|z_4)}] d\xi dy\\& + \iint p(z_4)p(y|z_4) log\, [p(y|\xi)\frac{p(y|z_4)}{p(y|\xi)}] dz_4dy\\
&=-\iint p(\xi)p(y|\xi) log\, \frac{p(y|\xi)}{p(y|z_4)} d\xi dy -\iint p(\xi)p(y|\xi) log\, p(y|z_4) d\xi dy\\& + \iint p(z_4)p(y|z_4) log\, \frac{p(y|z_4)}{p(y|\xi)} dz_4dy  + \iint p(z_4)p(y|z_4) log\, p(y|\xi) dz_4dy\\
&=-\int p(\xi)KL(p(y|\xi)||p(y|z_4))d\xi - \int p(y) log\, p(y|z_4)dy\\& +\int p(z_4)KL(p(y|z_4)||p(y|\xi))dz_4 + \int p(y) log\, p(y|\xi)dy\\
&= \int p(z_4)KL(p(y|z_4)||p(y|\xi))dz_4 + \int p(y) log\, \frac{p(y|\xi)}{p(y|z_4)}dy\\& -\int p(\xi)KL(p(y|\xi)||p(y|z_4))d\xi\\ 
&= \int p(z_4)KL(p(y|z_4)||p(y|\xi))dz_4 + \int p(\xi)p(y|\xi) log\, \frac{p(y|\xi)}{p(y|z_4)}dyd\xi \\& -\int p(\xi)KL(p(y|\xi)||p(y|z_4))d\xi\\
&= \int p(z_4)KL(p(y|z_4)||p(y|\xi))dz_4 + \int p(\xi)KL(p(y|\xi)||p(y|z_4))d\xi \\& -\int p(\xi)KL(p(y|\xi)||p(y|z_4))d\xi\\
&= \int p(z_4)KL(p(y|z_4)||p(y|\xi))dz_4 \\
&=\mathbb{E}_{z_4}[KL(p(y|z_4)||p(y|\xi))]\\
\end{split}
\end{equation}\par

Similarly, $I(y;z_2|\xi)=\mathbb{E}_{z_5}[KL(p(y|z_5)||p(y|\xi))]$, and $\frac{I(y;v_1|\xi)}{I(y;v_2|\xi)}$ can be calculated as:
\begin{equation}
\begin{split}
\frac{I(y;v_2|\xi)}{I(y;v_1|\xi)} &=\frac{\mathbb{E}_{z_5} [KL(p(y|z_5)\parallel p(y|\xi))]}{\mathbb{E}_{z_4} [KL(p(y|z_4)\parallel p(y|\xi))]}\\
&= \frac{1}{N} \sum_{n=1}^{N} \mathbb{E}_{\epsilon_1}\mathbb{E}_{\epsilon_2}\Bigl[\frac{KL(p(\hat{y}_2^n|\xi^n,z_2^n)||p(\hat{y}^n|\xi^n))}{KL(p(\hat{y}_1^n|\xi^n,z_1^n)||p(\hat{y}^n|\xi^n))}\Bigl]
\end{split}
\end{equation}
Finally, we have:
\begin{equation}
\begin{split}
    r&= 1 - tanh\Bigl(\ln \frac{1}{N} \sum_{n=1}^{N} \mathbb{E}_{\epsilon_1}\mathbb{E}_{\epsilon_2}\Bigl[\frac{KL(p(\hat{y}_2^n|\xi^n,z_2^n)||p(\hat{y}^n|\xi^n))}{KL(p(\hat{y}_1^n|\xi^n,z_1^n)||p(\hat{y}^n|\xi^n))}\Bigl]\Bigl)\\
    &=1-tanh(\ln \frac{I(y;v_2|\xi)}{I(y;v_1|\xi)})\propto\frac{I(y;v_1|\xi)}{I(y;v_2|\xi)}
\end{split}
\end{equation}
This completes the proof.
\end{proof}

\section{Proofs of \Cref{prop:inclusiveness} and \Cref{prop:exclusiveness}}\label{appendix:achievability}

As proposed in \Cref{sec:OMIB}, the objective function of MIB can be written as:
\begin{equation} \label{equ:IB alter-repeat}
\min_{\xi} \ell(\xi) = \min_{\xi} -I(\xi;y)+\beta (I(\xi;z_1)+r I(\xi;z_2))
\tag{17}
\end{equation} \par
Based on \Cref{Assumption 1.1} in \Cref{sec:Achieve OMIB}, we have:
 \begin{equation}
 \begin{split}
      &F(y)=\{a\}=\{a_0,a_1,a_2\},\\
      &F(v_1)=\{a_0,a_1,b_{1},b_{0}\},F(v_2)=\{a_0,a_2,b_{2},b_{0}\},\\ 
      &\{a_0\}\cap\{a_1\}=\emptyset,\{a_0\}\cap\{a_2\}=\emptyset, \{a_1\}\cap\{a_2\}=\emptyset,\\ 
      &\{b_{i}\}\cap\{a_0\}=\emptyset,\{b_{i}\}\cap\{a_1\}=\emptyset,\{b_{i}\}\cap\{a_2\}=\emptyset, \forall i \in \{0,1,2\}\\
      &I(y;v_1)=\{a\}\cap F(v_1)= \{a_0,a_1\},\\
      &I(y;v_2)=\{a\}\cap F(v_2)= \{a_0,a_2\}.\\
 \end{split}
 \nonumber 
 \end{equation}
\begin{definition}\label{Definition 1.2} 
The relative mutual information between encoding $z$ and task-relevant label $y$ is defined as the ratio of their mutual information to their total information:
\begin{equation}
    \widehat{I}(z;y)=\frac{I(z;y)}{F(z)\cup F(y)}=\frac{I(z;y)}{F(z,y)}=\frac{I(z;y)}{H(z)+H(y)-I(z;y)}
    \nonumber 
\end{equation}

\end{definition}
Compared to mutual information, relative mutual information more accurately reflects the amount of task-relevant information (i.e., $I(\xi;y)$) in total information (i.e., $F(\xi)\cup F(y)$), which aligns more with the objective of maximizing task-relevant information in MIB. Consequently, we replace $I(\xi;y) $ with $\widehat{I}(\xi;y)$ in \Cref{equ:IB alter-repeat} in the following analysis. 

\begin{lemma}[\textbf{Lemma \ref{prop:inclusiveness} restated}]\label{prop:inclusiveness-repeat}
Under \Cref{Assumption 1.1}, the objective function in \Cref{equ:IB alter-repeat} ensures:
\begin{equation}
    F(\xi)\supseteq\{a_0,a_1,a_2\},
\end{equation}
when $\beta \leq M_u$, where $M_u:=\frac{1}{(1+r) (H(v_1)+H(v_2)-I(v_1;v_2))}.$
\end{lemma}

\begin{proof}
Let $\{\check{\xi}_1\}=(\{a_0,a_1,a_2\}/(\{a_0,a_1,a_2\}\cap F(\xi)))\cap \{a_1\}$ represent the task-relevant information in $a_1$ that is not included in $\xi$. It is obvious:
\begin{equation}    
    \begin{aligned}
        &\{\check{\xi}_1\} \subset F(y),\{\check{\xi}_1\} \cap F(\xi)=\emptyset,\\
        &\{\check{\xi}_1\} \cap \{a_0\}= \emptyset,\{\check{\xi}_1\}\cap \{a_2\}= \emptyset, \\ 
        &\{\check{\xi}_1\}\cap F(v_2)=\emptyset, \{\check{\xi}_1\}\cap F(z_2)=\emptyset. 
    \end{aligned}
    \label{eq:def_check_z1} 
\end{equation}

If $\{\check{\xi}_1\} \neq \emptyset$, let $\xi'=\{\xi, \check{\xi}_1\}$. Using properties in \Cref{proof:prop},  we have:
\begin{equation}
\begin{split}
    I(\xi;v_1) - I(\xi';v_1) &= I(\xi;v_1) - I(\xi, \check{\xi}_1;v_1) \\
    &= I(\xi;v_1) - I(\xi;v_1) - \overbrace{I(v_1;\check{\xi}_1|\xi)}^{\{\check{\xi}_1\}\cap F(\xi)=\emptyset\eqref{eq:def_check_z1}}\\
    &= - I(v_1;\check{\xi}_1) <0
   \end{split}
   \nonumber 
\end{equation}\par

\begin{equation}
\begin{split}
I(\xi; v_2) - I(\xi'; v_2) &= I(\xi; v_2) - I(\xi, \check{\xi}_1; v_2) \\
&= I(\xi; v_2) - I(\xi; v_2) - \overbrace{I(v_2; \check{\xi}_1 \mid \xi)}^{\substack{\because \{\check{\xi}_1\} \cap F(v_2) = \emptyset\eqref{eq:def_check_z1} \\ \therefore I(v_2; \check{\xi}_1 \mid \xi) = 0}} \\
&= 0
\end{split}
\nonumber
\end{equation}\par

\begin{equation}
    \begin{split}
\widehat{I}(\xi';y)-\widehat{I}(\xi;y)
&=\frac{I(\xi,\check{\xi}_1;y)}{F(\xi,\check{\xi}_1,y)}-\frac{I(\xi;y)}{F(\xi,y)}\\
&= \frac{I(\xi,\check{\xi}_1;y)}{F(\xi,\check{\xi}_1,y)} - \frac{I(\xi;y)}{\underbrace{F(\xi,y)\cup F(\check{\xi}_1)}_{=F(\xi,y)\ \text{as}\ \{\check{\xi}_1\}\subset F(y)\eqref{eq:def_check_z1}}}\\
&= \frac{\overbrace{I(y;\check{\xi}_1|\xi)}^{\{\check{\xi}_1\}\cap F(\xi)=\emptyset\ \eqref{eq:def_check_z1}}}{F(\xi,\check{\xi}_1,y)}\\ 
&=\frac{I(y;\check{\xi}_1)}{F(\xi,\check{\xi}_1,y)} >0\\
\end{split}
\nonumber 
\end{equation}
For $\ell(\xi) - \ell(\xi')$, we have:
\begin{equation}
\begin{split}
   \ell(\xi) - \ell(\xi') &= \widehat{I}(\xi';y) - \widehat{I}(\xi;y) + \beta( I(\xi;v_1) -  I(\xi';v_1) +r I(\xi;v_2) - rI(\xi';v_2))\\
   &= \frac{I(y;\check{\xi}_1)}{F(\xi,\check{\xi}_1,y)} - \beta I(v_1;\check{\xi}_1)\\
   \end{split}
   \nonumber 
\end{equation}\par
When $\beta<\frac{I(y;\check{\xi}_1)}{I(v_1;\check{\xi}_1)F(\xi,\check{\xi}_1,y)}$, $\ell(\xi) - \ell(\xi') >0$, so optimizing the loss function will drive $\xi$ toward $\xi'$ until $\{\check{\xi}_1\}=\emptyset$, namely $F(\xi)\supseteq\{a_1\}$. We further suppose $\{\check{\xi}_2\}=(\{a_0,a_1,a_2\}/(\{a_0,a_1,a_2\}\cap F(\xi)))\cap \{a_2\}$ represent the task-relevant information in $a_2$ that is not included in $\xi$. Similarly, if $\{\check{\xi}_2\}\neq \emptyset$ and $\beta <\frac{I(y;\check{\xi}_2)}{r I(v_2;\check{\xi}_2)F(\xi,\check{\xi}_2,y)}$, the optimization will update $\xi$ until $\{\check{\xi}_2\}=\emptyset$, namely $F(\xi)\supseteq\{a_2\}$. \par
Moreover, let $\{\check{\xi}_0\}=(\{a_0,a_1,a_2\}/(\{a_0,a_1,a_2\}\cap F(\xi)))\cap \{a_0\}$ represent the task-relevant information in $a_0$ that is not included in $\xi$. If $\{\check{\xi}_0\} \neq \emptyset$, let $\xi'=\{\xi, \check{\xi}_0\}$. Then we have:
\begin{equation}
\begin{split}
    I(\xi;v_i) - I(\xi';v_i) &= I(\xi;v_i) - I(\xi,\check{\xi}_0;v_i)\\   &= I(\xi;v_i) - I(\xi;v_i) - \overbrace{I(v_i;\check{\xi}_0|\xi)}^{\{\check{\xi}_0\}\cap F(\xi)=\emptyset}\\
    &= - I(v_i;\check{\xi}_0) <0, \forall i \in \{1,2\}.
   \end{split}
   \nonumber 
\end{equation}
\begin{equation}
    \begin{split}
\widehat{I}(\xi';y)-\widehat{I}(\xi;y)
&=\frac{I(\xi,\check{\xi}_0;y)}{F(\xi,\check{\xi}_0,y)}-\frac{I(\xi;y)}{F(\xi,y)}\\
&= \frac{I(\xi,\check{\xi}_0;y)}{F(\xi,\check{\xi}_0,y)} - \frac{I(\xi;y)}{\underbrace{F(\xi,y)\cup F(\check{\xi}_0)}_{=F(\xi,y)\ \text{as}\ \{\check{\xi}_0\}\subset F(y)}}\\
&= \frac{\overbrace{I(y;\check{\xi}_0|\xi)}^{\{\check{\xi}_0\}\cap F(\xi)=\emptyset}}{F(\xi,\check{\xi}_0,y)}\\
&=\frac{I(y;\check{\xi}_0)}{F(\xi,\check{\xi}_0,y)} >0\\ 
\end{split}
\nonumber 
\end{equation}\par
For $\ell(\xi) - \ell(\xi')$, we have:
\begin{equation}
\begin{split}
   \ell(\xi) - \ell(\xi') &= \widehat{I}(\xi';y) - \widehat{I}(\xi;y) + \beta(I(\xi;v_1) - I(\xi';v_1) + r I(\xi;v_2) - r I(\xi';v_2))\\
   &= \frac{I(y;\check{\xi}_0)}{F(\xi,\check{\xi}_0,y)} - \beta(I(v_1;\check{\xi}_0) + r I(v_2;\check{\xi}_0))\\
   \end{split}
   \nonumber 
\end{equation}\par
Therefore,  when $\beta<\frac{I(y;\check{\xi}_0)}{F(\xi,\check{\xi}_0,y)(I(v_1;\check{\xi}_0) + r I(v_2;\check{\xi}_0))}$, the optimization will update $\xi$ until $\{\check{\xi}_0\}=\emptyset$, namely$F(\xi)\supseteq\{a_0\}$ .\par
Put together, the optimization procedure ensures $F(\xi)\supseteq\{a_0,a_1,a_2\}$ when:
\begin{equation}\label{eq:UB}
\beta <UB_{\beta}= min\Bigl(\frac{I(y;\check{\xi}_1)}{I(v_1;\check{\xi}_1)F(\xi,\check{\xi}_1,y)},\frac{I(y;\check{\xi}_2)}{r I(v_2;\check{\xi}_2)F(\xi,\check{\xi}_2,y)},\frac{I(y;\check{\xi}_0)}{F(\xi,\check{\xi}_0,y)(I(v_1;\check{\xi}_0) + r I(v_2;\check{\xi}_0))}\Bigr).
\end{equation} 
Finally, we prove in Lemma \ref{lem:2} below that $M_u=\frac{1}{(1+r) (H(v_1)+H(v_2)-I(v_1;v_2))}$ is a lower bound of $UB_{\beta}$. 
When $\beta < M_u$, the optimization procedure guarantees  $F(\xi)\supseteq\{a_0,a_1,a_2\}$. This completes the proof.
\end{proof}

\begin{lemma}\label{lem:2}
$UB_{\beta}$ in \Cref{eq:UB} satisfies:
$UB_{\beta} > M_u$, where $M_u=\frac{1}{(1+r) (H(v_1)+H(v_2)-I(v_1;v_2))}$. \par
$H(\cdot)$ and $I(\cdot;\cdot)$ can be estimated using MINE \cite{belghazi2018mine} (see Appendix \ref{net:estimate}).
\end{lemma}
\begin{proof}
As shown in \Cref{eq:UB}
\begin{equation}
\begin{split}
UB_{\beta}=min\Bigl(\frac{I(y;\check{\xi}_1)}{I(v_1;\check{\xi}_1)F(\xi,\check{\xi}_1,y)},\frac{I(y;\check{\xi}_2)}{r I(v_2;\check{\xi}_2)F(\xi,\check{\xi}_2,y)},\frac{I(y;\check{\xi}_0)}{F(\xi,\check{\xi}_0,y)(I(v_1;\check{\xi}_0) + r I(v_2;\check{\xi}_0))}\Bigr)
\end{split}
   \nonumber 
\end{equation}\par
$\because \{\check{\xi}_1\}\subseteq F(\xi,y)$, we have $F(\xi,\check{\xi}_1,y)= F(\xi,y)$ so that:
\begin{equation}
\begin{split}
   \frac{I(y;\check{\xi}_1)}{I(v_1;\check{\xi}_1)F(\xi,\check{\xi}_1,y)} &= \frac{I(y;\check{\xi}_1)}{I(v_1;\check{\xi}_1)F(\xi,y)}
   \end{split}
   \nonumber 
\end{equation}\par
$\because \{\check{\xi}_1\}\subseteq \{a_1\}, \{a_1\}\subseteq \{v_1\},\ \text{and}\ \{a_1\}\subseteq \{y\}$, $\therefore \{\check{\xi}_1\}\subseteq \{v_1\}\ \text{and}\ \{\check{\xi}_1\}\subseteq \{y\}$. Then, according to property \textbf{v} in \Cref{properties 1}, $I(y;\check{\xi}_1) = H(\check{\xi}_1)$ and $I(v_1;\check{\xi}_1) = H(\check{\xi}_1)$, which leads to:
\begin{equation}
\begin{split}
   \frac{I(y;\check{\xi}_1)}{I(v_1;\check{\xi}_1)F(\xi,y)} &= \frac{H(\check{\xi}_1)}{H(\check{\xi}_1)F(\xi,y)}\\
   &= \frac{1}{F(\xi,y)}
   \end{split}
   \nonumber 
\end{equation}\par
Similarly, $\frac{I(y;\check{\xi}_2)}{r I(v_2;\check{\xi}_2)F(\xi,\check{\xi}_2,y)}$ is simplify to $ \frac{1}{r F(\xi, y)}$. 
Moreover, $F(\xi,\check{\xi}_0,y)= F(\xi,y)$ since $\{\check{\xi}_0\}\subseteq F(\xi,y)$. Then, it follows that:
\begin{equation}
\begin{split}
   \frac{I(y;\check{\xi}_0)}{F(\xi,\check{\xi}_0,y)(I(v_1;\check{\xi}_0) + r I(v_2;\check{\xi}_0))} &= \frac{I(y;\check{\xi}_0)}{F(\xi,y)(I(v_1;\check{\xi}_0) + r I(v_2;\check{\xi}_0))} 
   \end{split}
   \nonumber 
\end{equation}\par

$\because \{\check{\xi}_0\}\subseteq \{a_0\}, \{a_0\}\subseteq \{v_1\}, \{a_0\}\subseteq\{v_2\},\ \text{and}\ \{a_0\}\subseteq\{y\}$, $\therefore \{\check{\xi}_0\}\subseteq \{v_1\}$, $\{\check{\xi}_0\}\subseteq \{v_2\}$, and $\{\check{\xi}_0\}\subseteq \{y\}$. Thus, by property \textbf{v} in \Cref{properties 1}, $I(y;\check{\xi}_0) = H(\check{\xi}_0)$, $I(v_1;\check{\xi}_0) = H(\check{\xi}_0)$, and $I(v_2;\check{\xi}_0) = H(\check{\xi}_0)$, which collectively lead to:
\begin{equation}
\begin{split}
   \frac{I(y;\check{\xi}_0)}{F(\xi,y)(I(v_1;\check{\xi}_0) + r I(v_2;\check{\xi}_0))}  &= \frac{H(\check{\xi}_0)}{F(\xi,y)(H(\check{\xi}_0) + r H(\check{\xi}_0))} \\
   &= \frac{1}{(1+r)F(\xi,y)}\\ 
   &< \min\Bigl(\frac{1}{F(\xi,y)},\frac{1}{r F(\xi,y)}\Bigr),\forall r>0\\
   &=\min\Bigl(\frac{I(y;\check{\xi}_1)}{I(v_1;\check{\xi}_1)F(\xi,\check{\xi}_1,y)} ,\frac{I(y;\check{\xi}_2)}{r I(v_2;\check{\xi}_2)F(\xi,\check{\xi}_2,y)}\Bigr).
   \end{split}
   \nonumber 
\end{equation}\par
Thus, we have:
\begin{equation}
\begin{split}
   UB_{\beta} &=\frac{I(y;\check{\xi}_0)}{F(\xi,y)(I(v_1;\check{\xi}_0) + r I(v_2;\check{\xi}_0))} \\
   &=\frac{1}{(1+r) F(\xi,y)} \\
   &> \frac{1}{(1+r)\underbrace{F(v_1,v_2)}_{\because F(\xi,y)\subset F(v_1,v_2)}}\\
   &=\frac{1}{(1+r) (H(v_1)+H(v_2)-I(v_1;v_2))}\\
   &=M_u
   \end{split}
   \nonumber 
\end{equation}\par 
This completes the proof.
\end{proof}
\begin{lemma}[\textbf{Lemma \ref{prop:exclusiveness} restated}]\label{prop:exclusiveness-repeat}
Under \Cref{Assumption 1.1}, the objective function in \Cref{equ:IB alter-2} is optimized when: 
\begin{equation}
F(\xi) \subseteq \{a_0,a_1,a_2\}\ 
\end{equation} 
\end{lemma}

\begin{proof}
Let $\hat{z}_1$ represent superfluous information that is specific to $v_1$ and not incorporated into $\xi$. Then, we have: 

\begin{equation}
    \begin{aligned}
        & \hat{z}_1\notin \{a_0,a_1,a_2\},\{\hat{z}_1\}\subset F(v_1), I(\hat{z}_1;v_1)>0,\\ 
        &I(\hat{z}_1;y)=0, \{\xi\}\cap \{\hat{z}_1\}=\emptyset, \{v_2\}\cap \{\hat{z}_1\}=\emptyset.
    \end{aligned}
\end{equation}

Let $\ddot{\xi}=\{\xi,\hat{z}_1\}$. The objective function becomes:

\begin{equation}
\begin{split}
   \ell(\ddot{\xi}) &= -\widehat{I}(\ddot{\xi};y)+\beta(I(\ddot{\xi};v_1)+r I(\ddot{\xi};v_2))\\ 
   &=-\widehat{I}(\xi,\hat{z}_1;y)+\beta( I(\xi,\hat{z}_1;v_1)+r I(\xi,\hat{z}_1;v_2))
   \end{split}
   \end{equation}\par
Then we have the following equations:
\begin{equation}
    \begin{split}
        I(\xi,\hat{z}_1;v_1)-I(\xi;v_1)&=\overbrace{I(v_1;\hat{z}_1|\xi)}^{\{\hat{z}_1\}\cap \{\xi\}=\emptyset}\\ 
        &=I(v_1;\hat{z}_1) >0
    \end{split}
\end{equation} 
\begin{equation}
    \begin{split}
        I(\xi,\hat{z}_1;v_2)-I(\xi;v_2)&=\overbrace{I(v_2;\hat{z}_1|\xi)}^{\{\hat{z}_1\}\cap \{v_2\}=\emptyset}\\ 
        &=0
    \end{split}
\end{equation} 
\begin{equation}
    \begin{split}
        \widehat{I}(\xi;y)-\widehat{I}(\xi,\hat{z}_1;y)&=\frac{I(\xi;y)}{F(\xi,y)}-\frac{I(\hat{z}_1,\xi;y)}{F(\hat{z}_1,\xi,y)}\\ 
        &=\frac{I(\xi;y)+\overbrace{I(y;\hat{z}_1|\xi)}^{=0\ \because I(\hat{z}_1;y)=0}}{F(\xi,y)}-\frac{I(\hat{z}_1,\xi;y)}{F(\hat{z}_1,\xi,y)}\\ 
        &=\frac{I(\hat{z}_1,\xi;y)}{F(\xi,y)}-\frac{I(\hat{z}_1,\xi;y)}{F(\hat{z}_1,\xi,y)}\\ 
        &=\frac{I(\hat{z}_1,\xi;y)}{F(\xi,y)}-\frac{I(\hat{z}_1,\xi;y)}{\underbrace{F(\hat{z}_1)+F(\xi,y)}_{\because \hat{z}_1\perp \{\xi,y\}}}\\ 
        &> 0
    \end{split}
\end{equation} \par

Put together, we have: 
\begin{equation}
    \begin{split}
        &\ell(\ddot{\xi})-\ell(\xi)\\
        &=(\widehat{I}(\xi;y)-\widehat{I}(\xi,\hat{z}_1;y))+\beta\Bigl((\widehat{I}(\hat{z}_1,\xi;v_1)- \widehat{I}(\xi;v_1))+r(\widehat{I}(\hat{z}_1,\xi;v_2)- \widehat{I}(\xi;v_2))\Bigr)\\
        &> 0,\ \text{if}\ \{\hat{z}_1\}\neq \emptyset.
    \end{split}
\end{equation}
For superfluous information $\hat{z}_2$ specific to $v_2$, we arrive at the same conclusion.
Finally, let $\hat{z}_0$ represent superfluous information that is shared by the two modalities and not encoded in $\xi$. Then, we have: 
\begin{equation}
    \begin{aligned}
        & \hat{z}_0 \notin \{a_0,a_1,a_2\},\{\hat{z}_0\}\subset F(v_1),\{\hat{z}_0\}\subset F(v_2),\\
        & I(\hat{z}_0;v_1)>0, I(\hat{z}_0;v_2)>0,\\ 
        &I(\hat{z}_0;y)=0, \{\xi\}\cap \{\hat{z}_0\}=\emptyset
    \end{aligned}
\end{equation}

Let $\ddot{\xi}=\{\xi,\hat{z}_0\}$. The objective function becomes:
\begin{equation}
\begin{split}
   \ell(\ddot{\xi}) &= -\widehat{I}(\ddot{\xi};y)+\beta(I(\ddot{\xi};v_1)+r I(\ddot{\xi};v_2))\\ 
   &=-\widehat{I}(\xi,\hat{z}_0;y)+\beta( I(\xi,\hat{z}_0;v_1)+r I(\xi,\hat{z}_0;v_2))
   \end{split}
   \end{equation}\par
Then we have the following equations:
\begin{equation}
    \begin{split}
        \widehat{I}(\xi;y)-\widehat{I}(\xi,\hat{z}_0;y)&=\frac{I(\xi;y)}{F(\xi,y)}-\frac{I(\hat{z}_0,\xi;y)}{F(\hat{z}_0,\xi,y)}\\ 
        &=\frac{I(\xi;y)+\overbrace{I(y;\hat{z}_0|\xi)}^{=0\ \because I(\hat{z}_0;y)=0}}{F(\xi,y)}-\frac{I(\hat{z}_0,\xi;y)}{F(\hat{z}_0,\xi,y)}\\ 
        &=\frac{I(\hat{z}_0,\xi;y)}{F(\xi,y)}-\frac{I(\hat{z}_0,\xi;y)}{F(\hat{z}_0,\xi,y)}\\ 
        &=\frac{I(\hat{z}_0,\xi;y)}{F(\xi,y)}-\frac{I(\hat{z}_0,\xi;y)}{\underbrace{F(\hat{z}_0)+F(\xi,y)}_{\because \hat{z}_0\perp \{\xi,y\}}}\\ 
        &> 0
    \end{split}
\end{equation} \par
\begin{equation}
    \begin{split}
        I(\xi,\hat{z}_0;v_1)-I(\xi;v_1)&=\overbrace{I(v_1;\hat{z}_0|\xi)}^{\{\hat{z}_0\}\cap \{\xi\}=\emptyset}\\ 
        &=I(v_1;\hat{z}_0) >0
    \end{split}
\end{equation} 
Similarly, we have $I(\xi,\hat{z}_0;v_2)-I(\xi;v_2) = I(v_2;\hat{z}_0) >0$. \par
Put together, we have: 
\begin{equation}
    \begin{split}
        &\ell(\ddot{\xi})-\ell(\xi)\\
        &=(\widehat{I}(\xi;y)-\widehat{I}(\xi,\hat{z}_0;y))+\beta\Bigl((\widehat{I}(\hat{z}_0,\xi;v_1)- \widehat{I}(\xi;v_1))+r(\widehat{I}(\hat{z}_0,\xi;v_2)- \widehat{I}(\xi;v_2))\Bigr)\\
        &> 0,\ \text{if}\ \{\hat{z}_0\}\neq \emptyset.
    \end{split}
\end{equation}
In a nutshell, the optimization procedure continues until $\xi$ does not encompass superfluous information, shared or modality specific, from $v_1$, $v_2$. That is, $F(\xi) \subseteq \{a_0,a_1,a_2\}$. This completes the proof.
\end{proof}

\section{Extension to Multiple Modalities}\label{proof:multi}

\begin{figure}[H]
\centering
\includegraphics[width=0.5\linewidth]{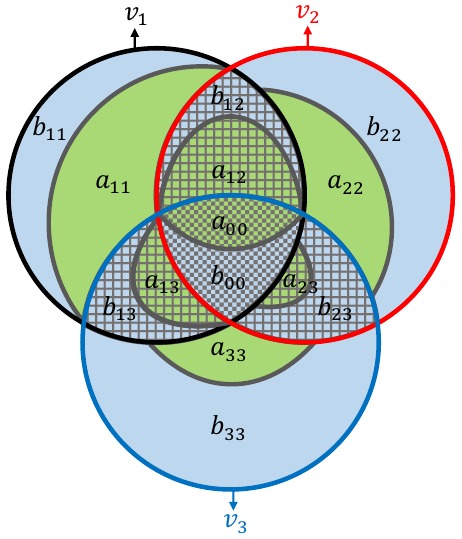}
\caption{Venn diagrams for three data modalities ($v_1$, $v_2$, and $v_3$). The gridded area represents consistent information, while the non-gridded area denotes modality-specific information. Task-relevant information is highlighted in green, whereas superfluous information is shown in blue.}
 \label{Fig: Information diagrams-2}
\end{figure}\par

The theoretical analysis of multiple modalities ($\geq 3$) is exemplified using three modalities, $v_1,v_2$, and $v_3$, which yet can be readily extended to more modalities. All mathematical notations remain consistent with those in \Cref{tab:notation} in \Cref{sec:notation}.  
\begin{assumption}\label{Assumption 3} 
Given three modalities, $v_1$, $v_2$, and $v_3$, the task-relevant information set $\{a\}$ consists of seven parts—$a_{00}, a_{11}, a_{22}, a_{33}, a_{12}, a_{13}, a_{23}$, as illustrated in \Cref{Fig: Information diagrams-2}. Specifically, $a_{00}$ is shared by all three modalities, while $a_{ij}$ is shared between modality pairs ($v_i,v_j$), $\forall i,j \in \{1,2,3\}$, $i<j$. Meanwhile, $a_{ii}$ is specific to $v_i$, $\forall i \in \{1,2,3\}$. The task-relevant labels $y$ are determined by $\{a\}$. On the other hand, superfluous information is represented by $\{b\}=\{b_{00},b_{11},b_{22},b_{33},b_{12},b_{13},b_{23}\}$. Here, $b_{00}$ is shared by all three modalities, while $b_{ij}$ is shared between modality pairs ($v_i,v_j$), $\forall i,j \in \{1,2,3\}$, $i<j$. Meanwhile, $b_{ii}$ is specific to $v_i$, $\forall i \in \{1,2,3\}$.
\end{assumption}
Based on the above assumption, the optimal MIB has the following definition:
\begin{definition}[\textbf{Optimal multimodal information bottleneck for three modalities}]\label{def:optimal MIB-2}
The optimal MIB for three modalities is defined as the MIB that encompasses all task-relevant information and free of superfluous information. Let $\xi_{opt-three}$ denote the optimal MIB, and it can be explicitly expressed as:
\begin{equation}
        F(\xi_{opt-three} )=\{a_{00}, a_{11}, a_{22}, a_{33}, a_{12}, a_{13}, a_{23}\}.
    \end{equation}
\end{definition}

In the following sections, we first demonstrate the method for achieving the optimal MIB, followed by a theoretical analysis to establish its theoretical foundation.\par
 
\subsection{Method}\label{multi method}
The warm-up and main training phases follow those for two modalities in \Cref{sec:method}, except for an additional modality $v_3$. The loss function $L_{TRB}$ remains the same for each modality as in \Cref{loss TRB}, while the loss function $L_{OMF}$ becomes: 
\begin{equation}\label{equ:IB loss imp 3}
\begin{split}
    L_{OMF} &= \frac{1}{N} \sum_{n=1}^{N} \mathbb{E}_{\epsilon_1}\mathbb{E}_{\epsilon_2}\mathbb{E}_{\epsilon_3} [-log\, q(y^n|\xi^n)] + \beta\Bigl(KL(p(\zeta_1^n|z_1^n)\parallel \mathcal{N}(0, I)\\& + r_1 KL(p(\zeta_2^n|z_2^n)\parallel \mathcal{N}(0, I)) + r_2 KL(p(\zeta_3^n|z_3^n)\parallel \mathcal{N}(0, I)\Bigl),
\end{split}
\end{equation}\par
where $\xi=CAN(\zeta_1,\zeta_2,\zeta_3,\theta_{CAN})$. Analogous to \Cref{equ:r} proved by \Cref{Lemma 2} in \Cref{sec:OMIB},  $r_1$ and $r_2$ are dynamic during training and explicitly calculated as:
\begin{small}
\begin{equation} \label{equ:r-3}
\begin{split}
    &r_1= 1 - tanh\Bigl(\ln \frac{1}{N} \sum_{n=1}^{N} \mathbb{E}_{\epsilon_1}\mathbb{E}_{\epsilon_2}\Bigl[\frac{KL(p(\hat{y}_2^n|\xi^n,z_2^n)||p(\hat{y}^n|\xi^n))}{KL(p(\hat{y}_1^n|\xi^n,z_1^n)||p(\hat{y}^n|\xi^n))}\Bigl]\Bigl),\\
    &r_2= 1 - tanh\Bigl(\ln \frac{1}{N} \sum_{n=1}^{N} \mathbb{E}_{\epsilon_1}\mathbb{E}_{\epsilon_3}\Bigl[\frac{KL(p(\hat{y}_3^n|\xi^n,z_3^n)||p(\hat{y}^n|\xi^n))}{KL(p(\hat{y}_1^n|\xi^n,z_1^n)||p(\hat{y}^n|\xi^n))}\Bigl]\Bigl).
\end{split}
\end{equation}
\end{small}\par
Moreover, as proposed in \Cref{lem:3} in \Cref{Multiple Theoretical Foundation}, when $\beta$ in \Cref{equ:IB loss imp 3} is upper-bounded by $M_u^2:=\frac{1}{(1+r_1+r_2) (\sum_{i=1}^{3} H(v_i)- \frac{2}{3}\sum_{1\leq i<j\leq3} I(v_i;v_j))}$,  the optimization of $L_{OMF}$ ensures the achievability of $\xi_{opt}$. 

\subsection{Theoretical Foundation}\label{Multiple Theoretical Foundation}
Under \Cref{Assumption 3}, we have:
\begin{equation}
 \begin{split}
      &F(y)=\{a\}=\{a_{00},a_{11},a_{22},a_{33},a_{12},a_{13},a_{23}\},\\
      &F(v_1)=\{a_{00},a_{11},a_{12},a_{13},b_{1},b_{12},b_{13},b_{0}\},\\
      &F(v_2)=\{a_{00},a_{22},a_{12},a_{23},b_{2},b_{12},b_{23},b_{0}\},\\
      &F(v_3)=\{a_{00},a_{33},a_{13},a_{23},b_{3},b_{13},b_{23},b_{0}\},\\
      &\{a_{ij}\}\cap\{a_{uv}\}=\emptyset,\, \forall a_{ij}, a_{uv}\in \{a\}, \text{where} \, a_{ij}\neq a_{uv}\\
      &\{b_{ij}\}\cap\{a_{uv}\}=\emptyset,\, \forall b_{ij} \in \{b\},\, a_{uv}\in \{a\}\\
      &I(y;v_1)=\{a\}\cap F(v_1)= \{a_{00},a_{11},a_{12},a_{13}\},\\
      &I(y;v_2)=\{a\}\cap F(v_2)= \{a_{00},a_{22},a_{12},a_{23}\}.\\
      &I(y;v_3)=\{a\}\cap F(v_3)= \{a_{00},a_{33},a_{13},a_{23}\}.\\
 \end{split}
 \nonumber 
 \end{equation}

Analogous to the analysis in \Cref{sec:OMIB}, the objective function for obtaining optimal MIB can be formulated as:
\begin{equation} \label{Multiple equ}
\max_{\xi} \ell(\xi) = \max_{\xi} I(\xi;y)-\beta (I(\xi;v_1)+r_1I(\xi;v_2)+r_2I(\xi;v_3)),
\end{equation}
which is equivalent to:
\begin{equation} \label{equ:IB multi alter}
\min_{\xi} \ell(\xi) = \min_{\xi} -I(\xi;y)+\beta (I(\xi;z_1)+r_1 I(\xi;z_2)+r_2 I(\xi;z_3)).
\end{equation}
The  \Cref{Lemma 1} in \Cref{sec:OMIB} can be trivially modified by adding the $r_2 I(\xi;z_3)$ term and applied here to establish $L_{OMF}$ in \cref{equ:IB loss imp 3} as a variational upper bound for  \Cref{equ:IB multi alter}.

\subsubsection{Achievability of optima information bottleneck for three modalities}

\begin{lemma}[\textbf{Inclusiveness of task-relevant information for three modalities}]
    \label{prop-3:inclusiveness}
Under \Cref{Assumption 3}, the objective function in \Cref{Multiple equ} ensures:
\begin{equation}
    F(\xi)\supseteq\{a_{00},a_{11},a_{22}, a_{33}, a_{12}, a_{13}, a_{23}\},
\end{equation}
when $\beta \leq M_u^2$, where $M_u^2:=\frac{1}{(1+r_1+r_2) (\sum_{i=1}^{3} H(v_i)- \frac{2}{3}\sum_{1\leq i<j\leq3} I(v_i;v_j))}$.

\end{lemma}

\begin{proof}
We first analyze under which condition $\xi$ can include all the task-relevant information specific to single modality.  Let $\{\check{\xi}_1\}=(\{a_{00},a_{11},a_{22}, a_{33}, a_{12}, a_{13}, a_{23}\}/(\{a_{00},a_{11},a_{22}, a_{33}, a_{12}, a_{13}, a_{23}\}\cap I(\xi)))\cap \{a_{11}\}$ represent the task-relevant information in $a_{11}$ that is not included in $\xi$. It is obvious:
\begin{equation}
    \begin{aligned}
        &\{\check{\xi}_1\}\subset F(y),\{\check{\xi}_1\} \cap F(\xi)=\emptyset,\\ 
        &\{\check{\xi}_1\} \cap \{a_{ij}\} = \emptyset, \quad \forall a_{ij} \in \{a\}/ \{a_{11}\},\\
        &\{\check{\xi}_1\}\cap F(v_2)=\emptyset, \{\check{\xi}_1\}\cap F(z_2)=\emptyset, \\
        &\{\check{\xi}_1\}\cap F(v_3)=\emptyset, \{\check{\xi}_1\}\cap F(z_3)=\emptyset
    \end{aligned}
    \nonumber
\end{equation}
 
If $\{\check{\xi}_1\}\neq \emptyset$,  let $\xi'=\{\xi,\check{\xi}_1\}$ and we have:
\begin{equation}
\begin{split}
  \ell(\xi)&=-\widehat{I}(\xi;y)+\beta ( I(\xi;v_1)+r_1 I(\xi;v_2)+r_2 I(\xi;v_3))\\
  &=-\widehat{I}(\xi;y)+\beta\Bigl( I(\xi;v_1)+ r_1 (I(v_2;\xi)+\overbrace{I(v_2;\check{\xi}_1|\xi)}^{=0,\, \because \{\check{\xi}_1\} \cap F(v_2)=\emptyset})+r_2 (I(v_3;\xi)+\overbrace{I(v_3;\check{\xi}_1|\xi)}^{=0,\, \because \{\check{\xi}_1\} \cap F(v_3)=\emptyset})\Bigr)\\
  &=-\widehat{I}(\xi;y)+\beta( I(\xi;v_1) +r_1 I(\xi,\check{\xi}_1;v_2) +r_2 I(\xi,\check{\xi}_1;v_3)),
   \end{split}
   \nonumber 
\end{equation}\par
\begin{equation}
\begin{split}
  \ell(\xi)-\ell(\xi')&=-\widehat{I}(\xi;y)+\beta( I(\xi;v_1) +r_1 I(\xi,\check{\xi}_1;v_2) +r_2 I(\xi,\check{\xi}_1;v_3))\\& - \Bigl(-\widehat{I}(\xi,\check{\xi}_1;y)+\beta( I(\xi,\check{\xi}_1;v_1) +r_1 I(\xi,\check{\xi}_1;v_2) +r_2 I(\xi,\check{\xi}_1;v_3))\Bigr)\\
  &=\widehat{I}(\xi,\check{\xi}_1;y)-\widehat{I}(\xi;y) + \beta( I(\xi;v_1) - I(\xi,\check{\xi}_1;v_1))
   \end{split}
   \nonumber 
\end{equation}\par
Using properties in \Cref{proof:prop},  we have:
\begin{equation}
    \begin{split}
\widehat{I}(\xi, \check{\xi}_1;y) - \widehat{I}(\xi;y) 
&=\frac{I(\xi, \check{\xi}_1;y)}{F(\xi) \cup\underbrace{(F(\check{\xi}_1)\cup F(y))}_{=F(y)\ \text{as}\ \{\check{\xi}_1\}\subset F(y)}}-\frac{I(\xi;y)}{F(\xi) \cup F(y)}\\
&=\frac{I(\xi, \check{\xi}_1;y)}{F(\xi,y)}-\frac{I(\xi;y)}{F(\xi,y)}\\
&=\frac{\overbrace{I(y;\check{\xi}_1|\xi)}^{\{\check{\xi}_1\}\cap \{\xi\}=\emptyset}}{F(\xi,y)}\\
&=\frac{I(y;\check{\xi}_1)}{F(\xi,y)}
\end{split}
\nonumber 
\end{equation}\par
\begin{equation}
\begin{split} 
   I(\xi;v_1) - I(\xi,\check{\xi}_1;v_1)
   &=- \overbrace{I(v_1;\check{\xi}_1|\xi)}^{\{\check{\xi}_1\}\cap \{\xi\}=\emptyset} \\
   &= - I(v_1;\check{\xi}_1)
\end{split}
\end{equation}\par
Thus:
\begin{equation}
\begin{split} 
   \ell(\xi) - \ell(\xi') &= \widehat{I}(\xi, \check{\xi}_1;y)-\widehat{I}(\xi;y)+\beta I(\xi;v_1) -\beta I(\xi, \check{\xi}_1;v_1) \\
   &=\frac{I(y;\check{\xi}_1)}{F(\xi,y)} - \beta I(v_1;\check{\xi}_1)
\end{split}
\end{equation}\par
When $\beta < \frac{I(y;\check{\xi}_1)}{F(\xi,y)I(v_1;\check{\xi}_1)}$, $\ell(\xi) - \ell(\xi')>0$. Therefore, optimizing the loss function will drive $\xi$ toward $\xi'$ until $\{\check{\xi}_1\}=\emptyset$, such that $F(\xi) \supseteq \{a_{11}\}$. We further suppose $\{\check{\xi}_2\}=(\{a_{00},a_{11},a_{22},a_{33},a_{12},a_{13},a_{23}\}/(\{a_{00},a_{11},a_{22},a_{33},a_{12},a_{13},a_{23}\}\cap I(\xi)))\cap \{a_{22}\}$ represent the task-relevant information in $a_{22}$ that is not included in $\xi$, and $\{\check{\xi}_3\}=(\{a_{00},a_{11},a_{22},a_{33},a_{12},a_{13},a_{23}\}/(\{a_{00},a_{11},a_{22},a_{33},a_{12},a_{13},a_{23}\}\cap I(\xi)))\cap \{a_{33}\}$ represent the task-relevant information in $a_{33}$ that is not included in $\xi$. Similarly, if $\{\check{\xi}_2\}\neq \emptyset$ and $\beta <\frac{I(y;\check{\xi}_2)}{r_1 I(v_2;\check{\xi}_2)F(\xi,\check{\xi}_2,y)}$, the optimization will update $\xi$ until $\{\check{\xi}_2\}=\emptyset$, leading to $F(\xi)\supseteq\{a_{22}\}$; and if $\{\check{\xi}_3\}\neq \emptyset$ and $\beta <\frac{I(y;\check{\xi}_3)}{r_2 I(v_3;\check{\xi}_3)F(\xi,\check{\xi}_3,y)}$, the optimization will update $\xi$ until $\{\check{\xi}_3\}=\emptyset$, leading to $F(\xi)\supseteq\{a_{33}\}$.\par

We then analyze under which condition $\xi$ can include all the task-relevant information shared by two modalities. Let $\{\check{\xi}_{12}\}=(\{a_{00},a_{11},a_{22},a_{33},a_{12},a_{13},a_{23}\}/(\{a_{00},a_{11},a_{22},a_{33},a_{12},a_{13},a_{23}\}\cap I(\xi)))\cap \{a_{12}\}$ represent the task-relevant information in $a_{12}$ that is not included in $\xi$. It is obvious:
\begin{equation}
    \begin{aligned}
        &\{\check{\xi}_{12}\}\subset F(y),\{\check{\xi}_{12}\} \cap F(\xi)=\emptyset,\\ 
        &\{\check{\xi}_{12}\} \cap \{a_{ij}\} = \emptyset, \quad \forall a_{ij} \in \{a\} / \{a_{12}\},\\ 
        &\{\check{\xi}_{12}\}\cap F(v_3)= \emptyset, \{\check{\xi}_{12}\}\cap F(z_3) = \emptyset
    \end{aligned}
    \nonumber
\end{equation}
If $\{\check{\xi}_{12}\}\neq \emptyset$, let $\xi'=\{\xi,\check{\xi}_{12}\}$, then we have:
\begin{equation}
\begin{split}
  \ell(\xi)&=-\widehat{I}(\xi;y)+\beta( I(\xi;v_1)+r_1 I(\xi;v_2)+r_2 I(\xi;v_3))\\
  &=-\widehat{I}(\xi;y)+\beta\Bigl( I(\xi;v_1)+ r_1 I(v_2;\xi)+r_2 (I(v_3;\xi)+\overbrace{I(v_3;\check{\xi}_{12}|\xi)}^{=0,\, \because \{\check{\xi}_{12}\} \cap F(v_3)=\emptyset})\Bigr)\\
  &=-\widehat{I}(\xi;y)+\beta( I(\xi;v_1) +r_1 I(\xi;v_2) +r_2 I(\xi,\check{\xi}_{12};v_3))
   \end{split}
   \nonumber 
\end{equation}\par
Therefore, $\ell(\xi)-\ell(\xi')$ can be written as:
\begin{equation}
\begin{split}
  \ell(\xi)-\ell(\xi')&=-\widehat{I}(\xi;y)+\beta( I(\xi;v_1) +r_1 I(\xi;v_2) +r_2 I(\xi,\check{\xi}_{12};v_3))\\& - \Bigl(-\widehat{I}(\xi,\check{\xi}_{12};y)+\beta( I(\xi,\check{\xi}_{12};v_1) +r_1 I(\xi,\check{\xi}_{12};v_2) +r_2 I(\xi,\check{\xi}_{12};v_3))\Bigr)\\
  &=\widehat{I}(\xi,\check{\xi}_{12};y)-\widehat{I}(\xi;y) + \beta\Bigl( I(\xi;v_1) - I(\xi,\check{\xi}_{12};v_1) + r_1(I(\xi;v_2) - I(\xi,\check{\xi}_{12};v_2))\Bigr)
   \end{split}
   \nonumber 
\end{equation}\par
Using properties in \Cref{proof:prop},  we have:
\begin{equation}
    \begin{split}
\widehat{I}(\xi, \check{\xi}_{12};y) - \widehat{I}(\xi;y) 
&=\frac{I(\xi, \check{\xi}_{12};y)}{F(\xi) \cup\underbrace{(F(y)\cup F(\check{\xi}_{12}))}_{=F(y)\ \text{as}\ \{\check{\xi}_{12}\}\subset F(y)}}-\frac{I(\xi;y)}{F(\xi) \cup F(y)}\\
&=\frac{I(\xi, \check{\xi}_{12};y)}{F(\xi,y)}-\frac{I(\xi;y)}{F(\xi,y)}\\
&=\frac{\overbrace{I(y;\check{\xi}_{12}|\xi)}^{\{\check{\xi}_{12}\}\cap \{\xi\}=\emptyset}}{F(\xi,y)}\\
&=\frac{I(y;\check{\xi}_{12})}{F(\xi,y)} > 0
\end{split}
\nonumber 
\end{equation}\par
\begin{equation}
\begin{split} 
   I(\xi;v_1) - I(\xi,\check{\xi}_{12};v_1)
   &=- \overbrace{I(v_1;\check{\xi}_{12}|\xi)}^{\{\check{\xi}_{12}\}\cap \{\xi\}=\emptyset} \\
   &= - I(v_1;\check{\xi}_{12}) < 0
\end{split}
\end{equation}\par
Similarly, we obtain that $I(\xi;v_2) - I(\xi,\check{\xi}_{12};v_2) = - I(v_2;\check{\xi}_{12})$.\par
Thus:
\begin{equation}
\begin{split} 
   \ell(\xi) - \ell(\xi') &= \widehat{I}(\xi,\check{\xi}_{12};y)-\widehat{I}(\xi;y) + \beta\Bigl( I(\xi;v_1) - I(\xi,\check{\xi}_{12};v_1) + r_1(I(\xi;v_2) - I(\xi,\check{\xi}_{12};v_2))\Bigr) \\
   &=\frac{I(y;\check{\xi}_{12})}{F(\xi,y)} - \beta( I(v_1;\check{\xi}_{12}) + r_1 I(v_2;\check{\xi}_{12})) 
\end{split}
\end{equation}\par

When $\beta < \frac{I(y;\check{\xi}_{12})}{F(\xi,y)(I(v_1;\check{\xi}_{12}) + r_1 I(v_2;\check{\xi}_{12}))}$, $\ell(\xi) - \ell(\xi')>0$. Therefore, optimizing the loss function will drive $\xi$ towards $\xi'$ until $\{\check{\xi}_{12}\}=\emptyset$, such that $F(\xi) \supseteq \{a_{12}\}$. Similarly, suppose that $\{\check{\xi}_{13}\}=(\{a_{00},a_{11},a_{22},a_{33},a_{12},a_{13},a_{23}\}/(\{a_{00},a_{11},a_{22},a_{33},a_{12},a_{13},a_{23}\}\cap I(\xi)))\cap \{a_{13}\}$ represents the task-relevant information in $a_{13}$ that is not included in $\xi$; and $\{\check{\xi}_{23}\}=(\{a_{00},a_{11},a_{22},a_{33},a_{12},a_{13},a_{23}\}/(\{a_{00},a_{11},a_{22},a_{33},a_{12},a_{13},a_{23}\}\cap I(\xi)))\cap \{a_{23}\}$ represents the task-relevant information in $a_{23}$ that is not included in $\xi$. Following the above procedure, we conclude that if $\{\check{\xi}_{13}\}\neq \emptyset$ and $\beta < \frac{I(y;\check{\xi}_{13})}{F(\xi,y)(I(v_1;\check{\xi}_{13}) + r_2 I(v_2;\check{\xi}_{13}))}$, the optimization will update $\xi$ until $\{\check{\xi}_{13}\}=\emptyset$, leading to $F(\xi)\supseteq\{a_{13}\}$; if $\{\check{\xi}_{23}\}\neq \emptyset$ and $\beta < \frac{I(y;\check{\xi}_{23})}{F(\xi,y)(r_1 I(v_1;\check{\xi}_{23}) + r_2 I(v_2;\check{\xi}_{23}))}$, the optimization will update $\xi$ until $\{\check{\xi}_{23}\}=\emptyset$, leading to $F(\xi)\supseteq\{a_{23}\}$.\par

Finally, we analyze under which condition $\xi$ can include all the task-relevant information shared by the three modalities. Let $\{\check{\xi}_{0}\}=(\{a_{00},a_{11},a_{22}, a_{33}, a_{12}, a_{13}, a_{23}\}/(\{a_{00},a_{11},a_{22}, a_{33}, a_{12}, a_{13}, a_{23}\}\cap I(\xi)))\cap \{a_{00}\}$ represent the task-relevant information in $a_{00}$ that is not included in $\xi$. It is obvious:
\begin{equation}
    \begin{aligned}
        &\{\check{\xi}_{0}\}\subset F(y),\{\check{\xi}_{0}\} \cap F(\xi)=\emptyset,\\ 
        &\{\check{\xi}_{0}\} \cap \{a_{ij}\} = \emptyset, \quad \forall a_{ij} \in \{a\} / \{a_{00}\},\\
        &\{\check{\xi}_{0}\} \cap F(v_l) = \emptyset, \quad \forall l \in \{1,2,3\}\\
    \end{aligned}
    \nonumber
\end{equation}
If $\{\check{\xi}_{0}\}\neq \emptyset$, let $\xi'=\{\xi,\check{\xi}_{0}\}$, and $\ell(\xi)-\ell(\xi')$ can be written as:
\begin{equation}
\begin{split}
  \ell(\xi)-\ell(\xi')&=-\widehat{I}(\xi;y)+\beta( I(\xi;v_1) +r_1 I(\xi;v_2) +r_2 I(\xi;v_3))\\& - \Bigl(-\widehat{I}(\xi,\check{\xi}_{0};y)+\beta( I(\xi,\check{\xi}_{0};v_1) +r_1 I(\xi,\check{\xi}_{0};v_2) +r_2 I(\xi,\check{\xi}_{0};v_3))\Bigr)\\
  &=\widehat{I}(\xi,\check{\xi}_{0};y)-\widehat{I}(\xi;y) + \beta\Bigl(I(\xi;v_1) - I(\xi,\check{\xi}_{0};v_1) \\&+ r_1(I(\xi;v_2) - I(\xi,\check{\xi}_{0};v_2)) + r_2(I(\xi;v_3) - I(\xi,\check{\xi}_{0};v_3))\Bigr)
   \end{split}
   \nonumber 
\end{equation}\par
Using properties in \Cref{proof:prop},  we have:
\begin{equation}
    \begin{split}
\widehat{I}(\xi, \check{\xi}_{0};y) - \widehat{I}(\xi;y) 
&=\frac{I(\xi, \check{\xi}_{0};y)}{F(\xi) \cup\underbrace{(F(y)\cup F(\check{\xi}_{0}))}_{=F(y)\ \text{as}\ \{\check{\xi}_{0}\}\subset F(y)}}-\frac{I(\xi;y)}{F(\xi) \cup F(y)}\\
&=\frac{I(\xi, \check{\xi}_{0};y)}{F(\xi,y)}-\frac{I(\xi;y)}{F(\xi,y)}\\
&=\frac{\overbrace{I(y;\check{\xi}_{0}|\xi)}^{\{\check{\xi}_{0}\}\cap \{\xi\}=\emptyset}}{F(\xi,y)}\\
&=\frac{I(y;\check{\xi}_{0})}{F(\xi,y)} > 0
\end{split}
\nonumber 
\end{equation}\par
\begin{equation}
\begin{split} 
   I(\xi;v_1) - I(\xi,\check{\xi}_{0};v_1)
   &=- \overbrace{I(v_1;\check{\xi}_{0}|\xi)}^{\{\check{\xi}_{0}\}\cap \{\xi\}=\emptyset} \\
   &= - I(v_1;\check{\xi}_{0}) < 0
\end{split}
\end{equation}\par
Similarly, we obtain that $I(\xi;v_2) - I(\xi,\check{\xi}_{0};v_2) = - I(v_2;\check{\xi}_{0})$ and $I(\xi;v_3) - I(\xi,\check{\xi}_{0};v_3) = - I(v_3;\check{\xi}_{0})$.\par
Thus:
\begin{equation}
\begin{split} 
   \ell(\xi) - \ell(\xi') &= \widehat{I}(\xi,\check{\xi}_{0};y)-\widehat{I}(\xi;y) + \beta\Bigl( I(\xi;v_1) - I(\xi,\check{\xi}_{0};v_1)\\& + r_1(I(\xi;v_2) - I(\xi,\check{\xi}_{0};v_2)) + r_2(I(\xi;v_3) - I(\xi,\check{\xi}_{0};v_3))\Bigr) \\
   &=\frac{I(y;\check{\xi}_{0})}{F(\xi,y)} - \beta( I(v_1;\check{\xi}_{0}) + r_1 I(v_2;\check{\xi}_{0}) + r_2 I(v_3;\check{\xi}_{0}))
\end{split}
\end{equation}\par
When $\beta < \frac{I(y;\check{\xi}_{0})}{F(\xi,y)(I(v_1;\check{\xi}_{0}) + r_1 I(v_2;\check{\xi}_{0}) + r_2 I(v_3;\check{\xi}_{0}))}$, $\ell(\xi) - \ell(\xi')>0$. Therefore, optimizing the loss function will drive $\xi$ toward $\xi'$ until $\{\check{\xi}_{0}\}=\emptyset$, such that $F(\xi) \supseteq \{a_{00}\}$.

Put together, the optimization procedure ensures $F(\xi)\supseteq\{a_{00},a_{11},a_{22},a_{33},a_{12},a_{13},a_{23}\}$ when:
\begin{equation}\label{eq:UB2}
\beta <UB_{\beta}:= min(UB^{1}_{\beta}, UB^{2}_{\beta}, UB^{3}_{\beta}, UB^{4}_{\beta}, UB^{5}_{\beta}, UB^{6}_{\beta}, UB^7_{\beta}).
\end{equation} 
where $UB^1_{\beta} = \frac{I(y;\check{\xi}_i)}{F(\xi,y)I(v_1;\check{\xi}_i)}$, $UB^2_{\beta} = \frac{I(y;\check{\xi}_2)}{r_1 F(\xi,y)I(v_2;\check{\xi}_2)}$, $UB^3_{\beta} = \frac{I(y;\check{\xi}_3)}{r_2 F(\xi,y)I(v_3;\check{\xi}_3)}$, $UB^4_{\beta}=\frac{I(y;\check{\xi}_{12})}{F(\xi,y)(I(v_1;\check{\xi}_{12}) + r_1 I(v_2;\check{\xi}_{12}))}$, $UB^5_{\beta}=\frac{I(y;\check{\xi}_{13})}{ F(\xi,y)(I(v_1;\check{\xi}_{13}) + r_2 I(v_2;\check{\xi}_{13}))}$, $UB^6_{\beta}=\frac{I(y;\check{\xi}_{23})}{ F(\xi,y)(r_1 I(v_1;\check{\xi}_{23}) + r_2 I(v_2;\check{\xi}_{23}))}$, and $UB^7_{\beta}=\frac{I(y;\check{\xi}_{0})}{F(\xi,y)(I(v_1;\check{\xi}_{0}) + r_1 I(v_2;\check{\xi}_{12}) + r_2 I(v_3;\check{\xi}_{0}))}$. \par
We complete the proof by proving that $M_u^2=\frac{1}{(1+r_1+r_2) (\sum_{i=1}^{3} H(v_i)- \frac{2}{3} \sum_{1\leq i < j\leq 3}I(v_i;v_j))}$ is a lower bound of $UB_{\beta}$ in Lemma \ref{lem:3} below. That is, when $\beta < M_u^2$, the optimization procedure guarantees $F(\xi)\supseteq\{a_{00},a_{11},a_{22},a_{33},a_{12},a_{13},a_{23}\}$. 
\end{proof}
\begin{lemma}\label{lem:3}  
$UB_{\beta}$ in \Cref{eq:UB2} satisfies:
$UB_{\beta} > M_u^2$, where $M_u^2=\frac{1}{(1+r_1+r_2) (\sum_{i=1}^{3} H(v_i)- \frac{2}{3} \sum_{1\leq i < j\leq 3}I(v_i;v_j))}$. \par
$H(\cdot)$ and $I(\cdot;\cdot)$ can be estimated using MINE \cite{belghazi2018mine} (see Appendix \ref{net:estimate}).
\end{lemma}\par
\begin{proof}
As shown in \Cref{eq:UB2}, $UB^1_{\beta} = \frac{I(y;\check{\xi}_1)}{F(\xi,y)I(v_1;\check{\xi}_1)}$. By property \textbf{v} in \Cref{properties 1}, $UB^1_{\beta}$ can be simplified as:
\begin{equation}
\begin{split}
   UB^1_{\beta} &= \frac{I(y;\check{\xi}_1)}{I(v_1;\check{\xi}_1)F(\xi,y)}\\ &= \frac{\overbrace{H(\check{\xi}_1)}^{=I(y;\check{\xi}_1), \because \{\check{\xi}_1\}\subseteq \{y\}}}{\underbrace{H(\check{\xi}_1)}_{=I(v_1;\check{\xi}_1), \because \{\check{\xi}_1\}\subseteq \{v_1\}}F(\xi,y)}\\
   &= \frac{1}{F(\xi,y)}
   \end{split}
   \nonumber 
\end{equation}\par
Similarly, we have $UB^2_{\beta} = \frac{I(y;\check{\xi}_2)}{r_1 F(\xi,y)I(v_2;\check{\xi}_2)} =\frac{1}{r_1 F(\xi,y)}$ and $UB^3_{\beta} = \frac{I(y;\check{\xi}_3)}{r_3 F(\xi,y)I(v_3;\check{\xi}_3)}=\frac{1}{r_2 F(\xi,y)}$.\par
$UB^4_{\beta}$ can be simplified as:
 \begin{equation}
     \begin{split}
      UB^4_{\beta}&=\frac{I(y;\check{\xi}_{12})}{F(\xi,y)(I(v_1;\check{\xi}_{12}) + r_1 I(v_2;\check{\xi}_{12}))}\\  &= \frac{\overbrace{H(\check{\xi}_{12})}^{=I(y;\check{\xi}_{12}), \because \{\check{\xi}_{12}\}\subseteq \{y\}}}{F(\xi,y)(\underbrace{H(\check{\xi}_{12})}_{=I(v_1;\check{\xi}_{12}), \because \{\check{\xi}_{12}\}\subseteq \{v_1\}} + r_1 \underbrace{H(\check{\xi}_{12})}_{=I(v_2;\check{\xi}_{12}), \because \{\check{\xi}_{12}\}\subseteq \{v_2\}})} \\
   &= \frac{1}{(1+r_1)F(\xi,y)}
     \end{split}
 \end{equation}
Similarly, we have $UB^5_{\beta}=\frac{I(y;\check{\xi}_{13})}{F(\xi,y)(I(v_1;\check{\xi}_{13}) + r_2 I(v_2;\check{\xi}_{13}))}=\frac{1}{(1+r_2)F(\xi,y)}$, and $UB^6_{\beta}=\frac{I(y;\check{\xi}_{23})}{F(\xi,y)(r_1 I(v_1;\check{\xi}_{23}) + r_2 I(v_2;\check{\xi}_{23}))}=\frac{1}{(r_1+r_2)F(\xi,y)}$. \par
$UB^7_{\beta}$ can be simplified as:
\begin{equation}
\begin{split}
   UB^7_{\beta}&=\frac{I(y;\check{\xi}_{0})}{F(\xi,y)(I(v_1;\check{\xi}_{0}) + r_1 I(v_2;\check{\xi}_{0}) + r_2 I(v_3;\check{\xi}_{0}))}\\  &= \frac{\overbrace{H(\check{\xi}_{0})}^{=I(y;\check{\xi}_{0}), \because \{\check{\xi}_{0}\}\subseteq \{y\}}}{F(\xi,y)(\underbrace{H(\check{\xi}_{0})}_{{=I(v_1;\check{\xi}_{0}), \because \{\check{\xi}_{0}\}\subseteq \{v_1\}}} + r_1 \underbrace{H(\check{\xi}_{0})}_{=I(v_2;\check{\xi}_{0}), \because \{\check{\xi}_{0}\}\subseteq \{v_2\}} + r_2 \underbrace{H(\check{\xi}_{0})}_{=I(v_3;\check{\xi}_{0}), \because \{\check{\xi}_{0}\}\subseteq \{v_3\}})} \\
   &= \frac{1}{(1+r_1+r_2)F(\xi,y)} \\
   \end{split}
   \nonumber 
\end{equation}\par

Therefore, for $\forall r_1, r_2>0$, we have:
\begin{equation}
\begin{split}
   \frac{1}{(1+r_1+r_2)F(\xi,y)}&< \min(\frac{1}{F(\xi,y)}, \frac{1}{r_1 F(\xi,y)}, \frac{1}{r_2 F(\xi,y)}, \frac{1}{(1+r_1)F(\xi,y)}, \frac{1}{(1+r_2)F(\xi,y)}, \frac{1}{(r_1+r_2)F(\xi,y)}),\\
   &=\min(UB^1_{\beta}, UB^2_{\beta}, UB^3_{\beta}, UB^4_{\beta}, UB^5_{\beta}, UB^6_{\beta}, UB^7_{\beta}) \\
 \Longrightarrow  UB_{\beta} &=\frac{1}{(1+r_1+r_2) F(\xi,y)} \\&> \frac{1}{(1+r_1+r_2) F(v_1,v_2,v_3)}\\
   &=\frac{1}{(1+r_1+r_2) (H(v_1)+H(v_2)+H(v_3)-I(v_1;v_2)-I(v_1;v_3)-I(v_2;v_3)+I(v_1;v_2;v_3))}\\
   \end{split}
   \nonumber 
\end{equation}\par
For the term \( I(v_1; v_2) + I(v_1; v_3) + I(v_2; v_3) - I(v_1; v_2; v_3) \), we have:  
\begin{equation}
\begin{split}
    &I(v_1; v_2) + I(v_1; v_3) + I(v_2; v_3) - I(v_1; v_2; v_3) < I(v_1; v_2) + I(v_1; v_3) + I(v_2; v_3) - I(v_1; v_2) = I(v_1; v_3) + I(v_2; v_3),  \\
    &I(v_1; v_2) + I(v_1; v_3) + I(v_2; v_3) - I(v_1; v_2; v_3) < I(v_1; v_2) + I(v_1; v_3) + I(v_2; v_3) - I(v_1; v_3) = I(v_1; v_2) + I(v_2; v_3),  \\
    &I(v_1; v_2) + I(v_1; v_3) + I(v_2; v_3) - I(v_1; v_2; v_3) < I(v_1; v_2) + I(v_1; v_3) + I(v_2; v_3) - I(v_2; v_3) = I(v_1; v_2) + I(v_1; v_3).
\end{split}
\nonumber
\end{equation}
To calculate \( I(v_1; v_2) + I(v_1; v_3) + I(v_2; v_3) - I(v_1; v_2; v_3) \), we sum up the individual inequalities, yielding: 
\begin{equation}
\begin{split}
I(v_1;v_2)+I(v_1;v_3)+I(v_2;v_3)-I(v_1;v_2;v_3)&<\frac{1}{3}(I(v_1;v_3)+I(v_2;v_3)+I(v_1;v_2)+I(v_2;v_3)+I(v_1;v_2)+I(v_1;v_3))\\&= \frac{2}{3}\sum_{1\leq i<j\leq3} I(v_i;v_j)
   \end{split}
   \nonumber 
\end{equation}\par
We then have:
\begin{equation}
\begin{split}
UB_{\beta} &>\frac{1}{(1+r_1+r_2) (H(v_1)+H(v_2)+H(v_3)-I(v_1;v_2)-I(v_1;v_3)-I(v_2;v_3)+I(v_1;v_2;v_3))}\\
   &>\frac{1}{(1+r_1+r_2) (\sum_{i=1}^{3} H(v_i)- \frac{2}{3}\sum_{1\leq i<j\leq3} I(v_i;v_j))} = M_u^2
   \end{split}
   \nonumber 
\end{equation}\par

This completes the proof.
\end{proof}

\begin{lemma}[\textbf{Exclusiveness of superfluous information for three modalities}]
    \label{prop-3:exclusiveness}
    Under \Cref{Assumption 3}, the objective function in \Cref{Multiple equ} is optimized when: \begin{equation}
F(\xi) \subseteq \{a_{00},a_{11},a_{22}, a_{33}, a_{12}, a_{13}, a_{23}\} 
\end{equation} 
\end{lemma}
\begin{proof}
We begin by analyzing the change in the loss function of our optimization after adding modality-specific superfluous information to $\xi$. Let $\hat{\xi}_1$ represent $v_1$-specific superfluous information that is not incorporated into $\xi$. Obviously: 

\begin{equation}
    \begin{aligned}
        & \{\hat{\xi}_1\}\notin \{a_{00},a_{11},a_{22}, a_{33}, a_{12}, a_{13}, a_{23}\},\{\hat{\xi}_1\}\subset F(v_1), I(\hat{\xi}_1;v_1)>0,\\ 
        &I(\hat{\xi}_1;y)=0, \{\xi\}\cap \{\hat{\xi}_1\}=\emptyset, \{v_2\} \cap \{\hat{\xi}_1\} =\emptyset, \{v_3\} \cap \{\hat{\xi}_1\}=\emptyset.
    \end{aligned}
\end{equation}

Let $\ddot{\xi}=\{\xi,\hat{\xi}_1\}$. The difference in loss function between $\xi$  and $\ddot{\xi}$ is computed as:

\begin{equation}
    \begin{split}
        \ell(\ddot{\xi})-\ell(\xi)
        &= -\widehat{I}(\ddot{\xi};y)+\beta(I(\ddot{\xi};v_1)+r_1 I(\ddot{\xi};v_2) + r_2 I(\ddot{\xi};v_3)) \\&- \Bigl(-\widehat{I}(\xi;y)+\beta(I(\xi;v_1)+r_1 I(\xi;v_2) + r_2 I(\xi;v_3))\Bigr)\\
        &=(\widehat{I}(\xi;y)-\widehat{I}(\xi,\hat{\xi}_1;y))+\beta\Bigl(\widehat{I}(\hat{\xi}_1,\xi;v_1)- \widehat{I}(\xi;v_1)\\&+r_1(\widehat{I}(\hat{\xi}_1,\xi;v_2)- \widehat{I}(\xi;v_2))+r_2(\widehat{I}(\hat{\xi}_1,\xi;v_3)- \widehat{I}(\xi;v_3))\Bigr),\\
    \end{split}
\end{equation}

Here we have:
\begin{equation}
    \begin{split}
        \widehat{I}(\xi;y)-\widehat{I}(\hat{\xi}_1,\xi;y)&=\frac{I(\xi;y)}{F(\xi,y)}-\frac{I(\hat{\xi}_1,\xi;y)}{F(\hat{\xi}_1,\xi,y)}\\ 
        &=\frac{I(\xi;y)+\overbrace{I(y;\hat{\xi}_1|\xi)}^{=0\ \because I(\hat{\xi}_1,y)=0}}{F(\xi,y)}-\frac{I(\hat{\xi}_1,\xi;y)}{F(\hat{\xi}_1,\xi,y)}\\ 
        &=\frac{I(\hat{\xi}_1,\xi;y)}{F(\xi,y)}-\frac{I(\hat{\xi}_1,\xi;y)}{F(\hat{\xi}_1,\xi,y)}\\ 
        &=\frac{I(\hat{\xi}_1,\xi;y)}{F(\xi,y)}-\frac{I(\hat{\xi}_1,\xi;y)}{\underbrace{F(\hat{\xi}_1)+F(\xi,y)}_{\because \hat{\xi}_1\perp \{\xi,y\}}}\\ 
        &>0
    \end{split}
\end{equation} \par
\begin{equation}
    \begin{split}
        I(\hat{\xi}_1,\xi;v_1)-I(\xi;v_1)&=\overbrace{I(v_1;\hat{\xi}_1|\xi)}^{\{\hat{\xi}_1\}\cap F(\xi)=\emptyset}\\ 
        &=I(v_1;\hat{\xi}_1) >0
    \end{split}
\end{equation} \par
\begin{equation}
    \begin{split}
        I(\hat{\xi}_1,\xi;v_2)-I(\xi;v_2)&=\overbrace{I(v_2;\hat{\xi}_1|\xi)}^{\{\hat{\xi}_1\}\cap F(v_2)=\emptyset}\\ 
        &=0
    \end{split}
\end{equation} \par
\begin{equation}
    \begin{split}
I(\hat{\xi}_1,\xi;v_3)-I(\xi;v_3) = 0
    \end{split}
\end{equation} \par

Thus, we have $ \ell(\ddot{\xi})-\ell(\xi)>0$, if $\{\hat{\xi}_1\}\neq \emptyset$.
For superfluous information $\hat{\xi}_2$ specific to $v_2$ and $\hat{\xi}_3$ specific to $v_3$, we arrive at the same conclusion. Next, we analyze the change in the loss function of our optimization after adding superfluous information shared by two modalities to $\xi$. Specifically, let $\hat{\xi}_{12}$ represent the superfluous information that is shared between modalities $v_1$ and $v_2$, and not incorporated into $\xi$. We have: 
\begin{equation}
    \begin{aligned}
        & \hat{\xi}_{12} \notin \{a_{00},a_{11},a_{22},a_{33},a_{12},a_{13},a_{23}\},\{\hat{\xi}_{12}\}\subset F(v_1),\{\hat{\xi}_{12}\}\subset F(v_2),\\
        &I(\hat{\xi}_{12};v_1)>0,I(\hat{\xi}_{12};v_2)>0,\\ 
        &I(\hat{\xi}_{12};y)=0, \{\xi\}\cap \{\hat{\xi}_{12}\}=\emptyset, \{v_3\}\cap \{\hat{\xi}_{12}\}=\emptyset.
    \end{aligned}
\end{equation}
Let $\ddot{\xi}=\{\xi,\hat{\xi}_{12}\}$. The difference in loss function between $\xi$  and $\ddot{\xi}$ is computed as:

\begin{equation}
    \begin{split}
        \ell(\ddot{\xi})-\ell(\xi)
        &= -\widehat{I}(\ddot{\xi};y)+\beta(I(\ddot{\xi};v_1)+r_1 I(\ddot{\xi};v_2) + r_2 I(\ddot{\xi};v_3)) \\&- \Bigl(-\widehat{I}(\xi;y)+\beta(I(\xi;v_1)+r_1 I(\xi;v_2) + r_2 I(\xi;v_3))\Bigr)\\
        &=(\widehat{I}(\xi;y)-\widehat{I}(\xi,\hat{\xi}_{12};y))+\beta\Bigl(\widehat{I}(\hat{\xi}_{12},\xi;v_1)- \widehat{I}(\xi;v_1)\\&+r_1(\widehat{I}(\hat{\xi}_{12},\xi;v_2)- \widehat{I}(\xi;v_2))+r_2(\widehat{I}(\hat{\xi}_{12},\xi;v_3)- \widehat{I}(\xi;v_3))\Bigr)\\
    \end{split}
\end{equation}
Here we have:
\begin{equation}
    \begin{split}
        \widehat{I}(\xi;y)-\widehat{I}(\hat{\xi}_{12},\xi;y)&=\frac{I(\xi;y)}{F(\xi,y)}-\frac{I(\hat{\xi}_{12},\xi;y)}{F(\hat{\xi}_{12},\xi,y)}\\ 
        &=\frac{I(\xi;y)+\overbrace{I(y;\hat{\xi}_{12}|\xi)}^{=0\ \because I(\hat{\xi}_{12},y)=0}}{F(\xi,y)}-\frac{I(\hat{\xi}_{12},\xi;y)}{F(\hat{\xi}_{12},\xi,y)}\\ 
        &=\frac{I(\hat{\xi}_{12},\xi;y)}{F(\xi,y)}-\frac{I(\hat{\xi}_{12},\xi;y)}{F(\hat{\xi}_{12},\xi,y)}\\ 
        &=\frac{I(\hat{\xi}_{12},\xi;y)}{F(\xi,y)}-\frac{I(\hat{\xi}_{12},\xi;y)}{\underbrace{F(\hat{\xi}_{12})+F(\xi,y)}_{\because \hat{\xi}_{12}\perp \{\xi,y\}}}\\ 
        &>0
    \end{split}
\end{equation} \par

\begin{equation}
    \begin{split}
        I(\hat{\xi}_{12},\xi;v_1)-I(\xi;v_1)&=\overbrace{I(v_1;\hat{\xi}_{12}|\xi)}^{\{\hat{\xi}_{12}\}\cap F(\xi)=\emptyset}\\ 
        &=I(v_1;\hat{\xi}_{12}) >0
    \end{split}
\end{equation} \par
\begin{equation}
    \begin{split}
        I(\hat{\xi}_{12},\xi;v_2)-I(\xi;v_2) = I(v_2;\hat{\xi}_{12}) >0
    \end{split}
\end{equation} \par
\begin{equation}
    \begin{split}
        I(\hat{\xi}_{12},\xi;v_3)-I(\xi;v_3)&=\overbrace{I(v_3;\hat{\xi}_{12}|\xi)}^{\{\hat{\xi}_{12}\}\cap F(v_3)=\emptyset}\\ 
        &=0
    \end{split}
\end{equation} \par
Thus, we have $ \ell(\ddot{\xi})-\ell(\xi)>0$, if  $\{\hat{\xi}_{12}\}\neq \emptyset$.
For superfluous information $\hat{\xi}_{13}$ shared between modalities $v_1$ and $v_3$, as well as $\hat{\xi}_{23}$ shared between modalities $v_2$ and $v_3$ , we arrive at the same conclusion. 
Finally, we analyze the change in the loss function of our optimization after adding superfluous information shared by all three modalities to $\xi$.  Let $\hat{\xi}_0$ represent the superfluous information shared by all three modalities and not incorporated into $\xi$. Then, we have: 
\begin{equation}
    \begin{aligned}
        & \hat{\xi}_0 \notin \{a_{00},a_{11},a_{22}, a_{33}, a_{12}, a_{13}, a_{23}\},\\
        &\{\hat{\xi}_0\}\subset F(v_1), \{\hat{\xi}_0\}\subset F(v_2),\{\hat{\xi}_0\}\subset F(v_3),\\
        &I(\hat{\xi}_0;v_1)>0,I(\hat{\xi}_0;v_2)>0,I(\hat{\xi}_0;v_3)>0,\\ 
        &I(\hat{\xi}_0;y)=0, \{\xi\}\cap \{\hat{\xi}_0\}=\emptyset.
    \end{aligned}
\end{equation}

Let $\ddot{\xi}=\{\xi,\hat{\xi}_{0}\}$. The difference in loss function between $\xi$  and $\ddot{\xi}$ is computed as:
\begin{equation}
    \begin{split}
        \ell(\ddot{\xi})-\ell(\xi)
        &= -\widehat{I}(\ddot{\xi};y)+\beta(I(\ddot{\xi};v_1)+r_1 I(\ddot{\xi};v_2) + r_2 I(\ddot{\xi};v_3)) \\&- \Bigl(-\widehat{I}(\xi;y)+\beta(I(\xi;v_1)+r_1 I(\xi;v_2) + r_2 I(\xi;v_3))\Bigr)\\
        &=(\widehat{I}(\xi;y)-\widehat{I}(\xi,\hat{\xi}_{0};y))+\beta\Bigl(\widehat{I}(\hat{\xi}_{0},\xi;v_1)- \widehat{I}(\xi;v_1)\\&+r_1(\widehat{I}(\hat{\xi}_{0},\xi;v_2)- \widehat{I}(\xi;v_2))+r_2(\widehat{I}(\hat{\xi}_{0},\xi;v_3)- \widehat{I}(\xi;v_3))\Bigr)\\
    \end{split}
\end{equation}
Here we have:
\begin{equation}
    \begin{split}
        \widehat{I}(\xi;y)-\widehat{I}(\hat{\xi}_0,\xi;y)&=\frac{I(\xi;y)}{F(\xi,y)}-\frac{I(\hat{\xi}_0,\xi;y)}{F(\hat{\xi}_0,\xi,y)}\\ 
        &=\frac{I(\xi;y)+\overbrace{I(y;\hat{\xi}_0|\xi)}^{=0\ \because I(\hat{\xi}_0,y)=0}}{F(\xi,y)}-\frac{I(\hat{\xi}_0,\xi;y)}{F(\hat{\xi}_0,\xi,y)}\\ 
        &=\frac{I(\hat{\xi}_0,\xi;y)}{F(\xi,y)}-\frac{I(\hat{\xi}_0,\xi;y)}{F(\hat{\xi}_0,\xi,y)}\\ 
        &=\frac{I(\hat{\xi}_0,\xi;y)}{F(\xi,y)}-\frac{I(\hat{\xi}_0,\xi;y)}{\underbrace{F(\hat{\xi}_0)+F(\xi,y)}_{\because \hat{\xi}_0\perp \{\xi,y\}}}\\ 
        &>0
    \end{split}
\end{equation} \par
\begin{equation}
    \begin{split}
        I(\hat{\xi}_0,\xi;v_1)-I(\xi;v_1)&=\overbrace{I(v_1;\hat{\xi}_0|\xi)}^{\{\hat{\xi}_0\}\cap F(\xi)=\emptyset}\\ 
        &=I(v_1;\hat{\xi}_0) >0
    \end{split}
\end{equation} \par
\begin{equation}
    \begin{split}
        I(\hat{\xi}_0,\xi;v_2)-I(\xi;v_2) = I(v_2;\hat{\xi}_0) > 0
    \end{split}
\end{equation} \par
\begin{equation}
    \begin{split}
        I(\hat{\xi}_0,\xi;v_3)-I(\xi;v_3) = I(v_3;\hat{\xi}_0) > 0
    \end{split}
\end{equation} \par 
Thus $\ell(\ddot{\xi})-\ell(\xi)>0$, if $\{\hat{\xi}_0\}\neq \emptyset$.
Put together, the optimization procedure continues until $\xi$ does not encompass superfluous information, specific to or shared by $v_1$, $v_2$, and $v_3$. That is, $F(\xi) \subseteq \{a_{00},a_{11},a_{22}, a_{33}, a_{12}, a_{13}, a_{23}\}$. This completes the proof.
\end{proof}

\begin{proposition}[\textbf{Achievability of optimal MIB for three modalities}]
    \label{prop:Ach OMIB-three}
\Cref{prop-3:inclusiveness}, and \Cref{prop-3:exclusiveness} jointly demonstrate that the optimal MIB $\xi_{opt-three}$ is achievable through optimization of \Cref{Multiple equ}  with $\beta \in (0, M^2_u]$.
\end{proposition}
\begin{proof}
   From \Cref{prop-3:inclusiveness} and \Cref{prop-3:exclusiveness}, we have $F(\xi)\supseteq\{a_{00},a_{11},a_{22},a_{33},a_{12},a_{13},a_{23}\}$ if $\beta \in (0, M^2_u]$, and $F(\xi) \subseteq \{a_{00},a_{11},a_{22},a_{33},a_{12},a_{13},a_{23}\}$, respectively. Thus, $F(\xi)=\{a_{00},a_{11},a_{22},a_{33},a_{12},a_{13},a_{23}\}$, which corresponds to $\xi_{opt-three}$ in \Cref{def:optimal MIB-2}.
\end{proof}

To expedite the training process, we can also set $M_u^2=\frac{1}{5 (\sum_{i=1}^{3} H(v_i)- \frac{2}{3} \sum_{1\leq i < j\leq 3}I(v_i;v_j))}$ as an upper bound and $M_l^2:=\frac{1}{5 (\sum_{i=1}^{3} H(v_i))}$ as a lower bound for $\beta$, resulting in $\beta \in [M_l^2, M_u^2]$.

\section{Estimation of Mutual Information and Information Entropy}\label{net:estimate}
We apply the Mutual Information Neural Estimation (MINE) method \cite{belghazi2018mine} to estimate the information entropy of each data modality and the mutual information between data modalities. Given two modalities $X$ and $Z$, MINE employs a neural network, implemented as a two-layer Multi-Layer Perceptron (MLP) network with ReLU activation function  \cite{belghazi2018mine}, to learn a set of functions \( \{T_{\theta}\}_{\theta \in \Theta} \). Each function \( T_{\theta}: X \times Z \to \mathbb{R} \) maps sample pairs to real values, enabling mutual information estimation as:
\begin{equation}
I(X;Z) = \sup_{\theta \in \Theta} \mathbb{E}_{P_{XZ}}[T_{\theta}] - \log \mathbb{E}_{P_X \otimes P_Z}[e^{T_{\theta}}].
\end{equation}
Here, \( \mathbb{E}_{P_{XZ}}[T_{\theta}] \) represents the expected value of \( T_{\theta} \) calculated using sample pairs from the joint distribution \( P_{XZ} \), and \( \mathbb{E}_{P_X \otimes P_Z}[e^{T_{\theta}}] \) represents the expected value of \( T_{\theta} \) calculated using sample pairs from the product of marginal distribution \( P_X \otimes P_Z \). The joint distribution \( P_{XZ} \) is approximated using matched sample pairs $(X,Z)$, while \( P_X \otimes P_Z \) is approximated using perturbed pairs $(X,Z')$, where $Z'$ is obtained by shuffling $Z$. The information entropy \( H(X) \) is computed as the mutual information of \( X \) with itself:
\begin{equation}
H(X) = I(X;X)
\end{equation}\par
Specifically, in this case, $Z$ and $Z'$ are replaced by $X$ and $X'$, where $X'$ is obtained by shuffling $X$. 

\section{Synthetic Data} \label{gen_sg}
Following \cite{xue2023the}, we simulate pairs of Gaussian observations and task labels, {$x_1 \in \mathbb{R}^{d_1}$, $x_2 \in \mathbb{R}^{d_2}$;$y$}, where $x_1$ and $x_2$ represent observations from two modalities with dimensionalities $d_1$ and $ d_2$, respectively, and $ y\in\{0,1\}$ represents the corresponding binary label. The feature vectors of $x_1$ and $x_2$ are defined as:
\begin{equation}
\begin{split} 
   & x_1 = [b_0;b_1;a_0;a_1], x_2 = [b_0;b_2;a_0;a_2], \\
\end{split}
\end{equation}
where
\begin{itemize}
    \item $a_0 \in \mathbb{R}^{d_0} \sim \mathcal{N}(0, I_{d_0})$ denotes consistent, task-relevant information shared between the modalities;
    \item $a_1\in \mathbb{R}^{d_{11}} \sim \mathcal{N}(0, I_{d_{11}})$ and $a_2\in \mathbb{R}^{d_{21}} \sim \mathcal{N}(0, I_{d_{21}})$ represent modality-specific, task-relevant information; 
    \item $b_0 \in \mathbb{R}^{d'_0} \sim \mathcal{N}(0, I_{d'_0})$ is consistent, superfluous information;
    \item $b_1\in \mathbb{R}^{d_{12}} \sim \mathcal{N}(0, I_{d_{12}})$ and $b_2 \in \mathbb{R}^{d_{22}} \sim \mathcal{N}(0, I_{d_{22}})$ are modality-specific, superfluous information.
\end{itemize}

Here, \( \mathcal{N}(0, I_d) \) denotes a multivariate Gaussian distribution with mean 0 and identity covariance matrix \( I \) of dimensionality $d$. 
The dimensions satisfy:
\begin{equation}
    d_1 = d_0 + d'_0 + d_{11} + d_{12}, d_2 = d_0 + d'_0 + d_{21} + d_{22}
\end{equation}

The label $y\in \{0,1\}$ is generated using a Dirac function $\Delta$, which depends solely on the task-relevant information $a_0,a_1$, and $a_2$:
\begin{equation}
   y :=\Delta(\langle \delta,[a_0;a_1;a_2] \rangle >0)
\end{equation}
where $\delta \in \mathbb{R}^{d_0+d_{11}+d_{21}} \sim \mathcal{N}(0,I_{d_0+d_{11}+d_{21}})$ is a randomly sampled vector serving as a separating hyperplane, and $\langle\cdot,\cdot\rangle$ denotes the inner product operation. 

By adjusting $d_0$, $d_{11}$, and $d_{21}$, we can control the distribution of task-relevant information across the two modalities, enabling the simulation of imbalanced task-relevant information. Specifically, as illustrated in \Cref{Fig: exp1}, we simulate three SIM datasets (SIM-\{I-III\}) to be used in three experimental cases, respectively (see  \Cref{exp:Synthetic Gaussian}). Firstly, for all cases, we set $d_0=d'_0=200$. For SIM-I used in \textbf{case i}, we set $d_{11}(500)\gg d_{21} (100)$ so that $a_1$ has a significantly greater impact on determining $y$, compared to $a_2$. This configuration implies that Modality I dominates Modality II in terms of task-relevant information. For SIM-II used in \textbf{case ii}, we switch the setting of $d_{11}$ and $d_{12}$, making Modality II dominant over Modality I. Finally, for SIM-III used in \textbf{case iii}, we set $d_{11}=d_{12}=300$ to ensure both modalities contribute equally to task-relevant information.

\section{Detailed Dataset Description}\label{dataset:description}
\paragraph{SIM.} See Appendix \ref{gen_sg}.

\paragraph{CREMA-D.} CREMA-D is an audio-visual dataset designed to study multimodal emotional expression and perception \cite{cao2014crema}. It captures actors portraying six basic emotional states—happy, sad, anger, fear, disgust, and neutral—through facial expressions and speech.

\paragraph{CMU-MOSI.} CMU-MOSI  \cite{zadeh2016multimodal} consists of 93 videos, from which 2,199 utterance are generated, each containing an image, audio, and language component. Each utterance is labeled with sentiment intensity ranging from -3 to 3. 

\paragraph{10x-hNB-\{A-H\}\& 10x-hBC-\{A-D\}.} 
The 10x-hNB-\{A-H\} datasets comprises eight datasets derived from healthy human breast tissues, while the 10x-hBC-\{A-D\} datasets contain four datasets from human breast cancer tissues \cite{xu2024detecting}. As shown in \Cref{Fig: task:abnormal}, each dataset corresponds to a tissue section and include gene expression and histology modalities. For each tissue section, gene expression profiles (i.e., gene read counts) are measured at fixed spatial spots across the section. During data preprocessing, genes detected in fewer than 10 spots are excluded, and raw gene expression counts are normalized by library size, log-transformed, and reduced to the 3,000 highly variable genes (HVGs) using the SCANPY package~\cite{wolf2018scanpy, li2024dual, xu2024domain,du2025methodological}. The corresponding histology image is segmented into $32\times 32$ region patches centered around each spatial spot, from which pathological patches are identified for anomaly detection. OMIB and baseline models are trained on the 10x-hBC-\{A-H\} datasets to learn multimodal representations of normal tissue regions within a compact hypersphere in the latent space.  The trained models are then applied to the 10x-hBC-\{A-D\} datasets during inference.

\begin{table*}[t]
\centering
    \fontsize{7.8pt}{10.5pt}\selectfont
    \renewcommand{\arraystretch}{1.25}
    \setlength{\tabcolsep}{3.55pt} 
\caption{Overview of the experimental datasets. }
\label{table:data}
\begin{tabular}{>{\centering\arraybackslash}p{1.7cm}|>{\centering\arraybackslash}p{1.7cm}|>{\centering\arraybackslash}p{4.8cm}}
\Xhline{1pt}
Dataset & Type &Number of samples (Anomaly proportion)\\ \Xhline{0.7pt}
SIM-\{I-III\}       &Training  & 9000    \\ \Xhline{0.5pt}
SIM-\{I-III\}       &Test      & 1000    \\ \Xhline{0.5pt}
CREMA-D   &Training  & 6,698   \\ \Xhline{0.5pt}
CREMA-D   &Test     & 744  \\ \Xhline{0.5pt}
MOSI      &Training & 1281  \\ \Xhline{0.5pt}
MOSI      &Test & 685   \\ \Xhline{0.5pt}
10x-hNB-A & Training  & 2364  \\ \Xhline{0.5pt}
10x-hNB-B & Training  & 2504  \\ \Xhline{0.5pt}
10x-hNB-C & Training  & 2224 \\ \Xhline{0.5pt}
10x-hNB-D & Training  & 3037  \\ \Xhline{0.5pt}
10x-hNB-E & Training  & 2086  \\ \Xhline{0.5pt}
10x-hNB-F & Training  & 2801   \\ \Xhline{0.5pt}
10x-hNB-G & Training  & 2694   \\ \Xhline{0.5pt}
10x-hNB-H & Training  & 2473  \\ \Xhline{0.5pt}
10x-hBC-A & Test  & 346 (12.43\%) \\ \Xhline{0.5pt}
10x-hBC-B & Test  & 295 (78.64\%) \\ \Xhline{0.5pt}
10x-hBC-C & Test  & 176 (27.84\%) \\ \Xhline{0.5pt}
10x-hBC-D & Test  & 306 (54.58\%) \\
\Xhline{1pt}
\end{tabular}
\label{table:dataset}
\end{table*}

\section{Detailed Network Architecture Implementation}\label{net:Implem} 

\paragraph{Modality-specific encoder.} 
We implement the encoder as follows: 
\begin{itemize}
    \item The SIM datasets: A two-layer MLP with the GELU activation function, outputting 256-dimensional embeddings.
    \item The CREMA-D dataset: Both video and audio encoders use ResNet-18, producing 512-dimensional outputs.
    \item The CMU-MOSI dataset: Conv1D is employed for both the audio and visual modalities, while BERT is utilized for the textual modality, with all three encoders producing 512-dimensional embeddings.
    \item The 10x-hNB-\{A-H\}\& 10x-hBC-\{A-D\} datasets: ResNet-18 and a two-layer graph convolutional network are used for the histology and gene expression modalities, respectively, with both encoders producing 256-dimensional embeddings.
\end{itemize}

\paragraph{Task-relevant prediction head.} We implement task-relevant prediction head as follows:
\begin{itemize}
    \item The SIM and CREMA-D datasets: The prediction head is implemented as a single linear layer (input X 512 X 100) followed by a softmax layer for classification, producing a $k$-dimensional output, where $k$ is the number of classification types. The TRB loss $L_{TRB}$ is cross-entropy;
    \item The CMU-MOSI dataset: The prediction head is implemented as a single linear layer MLP (input X 50 X 1), outputting a single real value. $L_{TRB}$ is mean squared error;
    \item The 10x-hNB-\{A-H\} and 10x-hBC-\{A-D\} datasets: The prediction head is implemented under the SVDD framework \cite{pmlr-v80-ruff18a,xu2025meatrd} as a two-layer MLP (input X 256 X 256) with LeakyReLU activation functions, producing 256-dimensional latent multimodal representations. $L_{TRB}$ is defined as: 
    \begin{equation}
        \begin{split}   
            L_{TRB} &= \frac{1}{N} \sum_{i=1}^N \| \hat{y} - c \|^2 + \lambda \cdot \mathcal{R}(\Theta), \\
            &\, \, \, \, \,c=\frac{1}{N}\sum_{i=1}^{N} \hat{y},
        \end{split}
    \end{equation}
    where $\hat{y}$ denotes the output of the prediction head, $c$ the center of the hypersphere,$\mathcal{R}(\Theta)$ the function that regularizes model parameters $\Theta$ for reducing model complexity and preventing model collapse, $\lambda$ is the regularization weight. 
\end{itemize}

\paragraph{Variational Autoencoder.} The VAE is implemented as two-layer MLP with two heads, outputting the $\mu$ and $\Sigma$, respectively.

\paragraph{Cross-Attention Network.} For datasets with two modalities, the cross-attention is implemented as: 
\begin{equation}\label{attn}
\xi = \mathrm{Attn}\left([\zeta_{1} \Vert \zeta_{2}]; W_{Q}, W_{K}, W_{V}\right)
\end{equation}  
where $\mathrm{Attn}$ represents the standard attention block, $W_Q$, $W_K$, and $W_V$ denote learnable projection matrices for queries, keys, and values respectively. The operator $\Vert$ represents concatenation along the feature dimension.

For datasets with three modalities, the cross-attention is extended as:  
\begin{equation}\label{attn:2}
\xi = \mathrm{Attn}\left([\zeta_{1} \Vert \zeta_{2} \Vert \zeta_{3}]; W_{Q}, W_{K}, W_{V}\right)
\end{equation}  \par
Finally, a Linear layer is applied to map $\xi$ back to the same dimensions as $\zeta_1$ and $\zeta_2$.  

\section{Experimental Settings} \label{exp_set}
All experiments are implemented using PyTorch \cite{paszke2019pytorch}, with the following settings: \par
\paragraph{SIM.} We use the Adam optimizer with a learning rate of 1e-4 and train the model for 100 epochs. The dataset consists of 10,000 samples, split into training and test sets with a 9:1 ratio. 
\paragraph{CREMA-D.} The model is trained using the SGD optimizer with a batch size of 64, momentum of 0.9, and weight decay of 1e-4. The learning rate is initialized at 1e-3 and decays by a factor of 0.1 every 70 epochs, reaching a final value of 1e-4.  The dataset is divided into a training set containing 6,698 samples and a test set of 744 samples.
\paragraph{CMU-MOSI.} We employ the Adam optimizer with a learning rate of 1e-5. 
All other hyperparameters and settings follow \cite{9767641}. 2,199 utterances are extracted from the dataset, which are split into a training set (1,281 samples) and a test set (685 samples).
\paragraph{10x-hNB-\{A-H\}\& 10x-hBC-\{A-D\}.} We use the Adam optimizer with a learning rate of 1e-4 and a weight decay of 0.1. The training batch size is set to 128. The final multimodal representation has a dimensionality of 256.

\begin{figure}[hpt]
\centering
\includegraphics[width=0.5\linewidth]{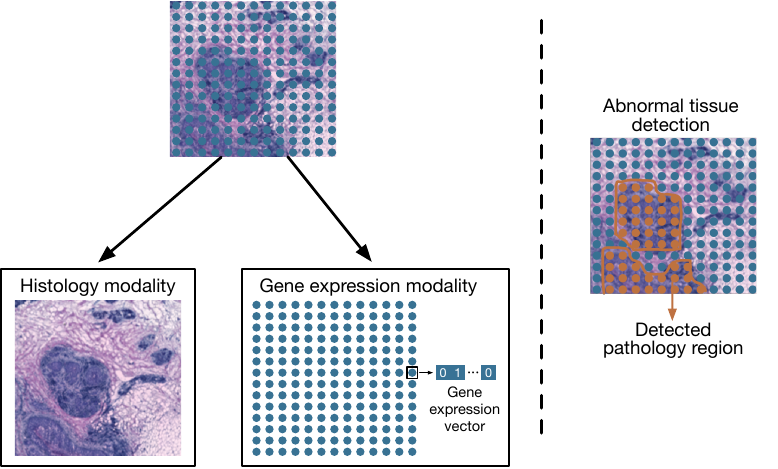}
\caption{Genomic multi-modal applications. Genomic data can be divided into two modalities: the histology modality and the gene expression modality. The histology modality comprises tissue image, while the gene expression modality consists of gene expression vectors, where each spot corresponds to a vector composed of multiple gene expression values. These two modalities are spatially aligned through shared spatial information. By integrating and analyzing both modalities, abnormal regions within the tissue can be effectively detected.
}
 \label{Fig: task:abnormal}
\end{figure}\par
\section{Benchmark Methods} \label{benchmark}
Here, we briefly describe the eight benchmark methods used in this study. For non-MIB-based methods:
\begin{itemize}
    \item Concat refers to simple concatenation of multi-modal features, which yet is the most widely used fusion approach. 
    \item BiGated \cite{kiela2018efficient} flexibly integrates information from different modalities through a gating mechanism.
    \item MISA \cite{hazarika2020misa} decomposes data into modality-invariant and modality-specific representations, using alignment and divergence constraints for better multimodal representation.
\end{itemize}

For MIB-based methods:
\begin{itemize}
    \item deep IB \cite{wang2019deep} extends VIB to a multi-view setting, maximizing mutual information between labels and the joint representation while minimizing mutual information between each view's latent representation and the original data;
    \item MMIB-Cui \cite{10413604} addresses the issues of modality noise and modality gap in multimodal named entity recognition (MNER) and multimodal relation extraction (MRE) by integrating the information bottleneck principle, thereby enhancing the semantic consistency between textual and visual information;
    \item MMIB-Zhang \cite{zhang2022information} effectively controls the learning of multimodal representations by imposing mutual information constraints between different modality pairs, removing task-irrelevant information within single modalities while retaining relevant information, significantly improving performance in multimodal sentiment analysis;
    \item DMIB \cite{Fang2024Dynamic} effectively filters out irrelevant information and noise, while introducing a sufficiency loss to retain task-relevant information, demonstrating significant robustness in the presence of redundant data and noisy channels.
    \item E-MIB, L-MIB, and C-MIB \cite{9767641} aim to learn effective multimodal and unimodal representations by maximizing task-relevant mutual information, eliminating modality redundancy, and filtering noise, while exploring the effects of applying MIB at different fusion stages.
\end{itemize}

\section{Evaluation Metrics} \label{evaluation}
In Emotion Recognition, we use accuracy (Acc) as the evaluation metric. For Multimodal Sentiment Analysis, we use the mean absolute error (MAE) and Pearson's correlation coefficient (Corr) to evaluate the predicted scores against the true scores. Additionally, as sentiment intensity scores can be divided into positive and negative categories, F1-score and polarity accuracy (Acc-2) are also utilized to evaluate prediction results as a binary classification task. Additionally, the interval of $[-3,3]$ contains seven integer scores to which each predicted score is neared to. This allows the using of categorical accuracy (Acc-7) to evaluate the prediction results. Finally, for the Anomalous Tissue Detection task, we evaluate performance using AUC score and F1-score. The AUC score is calculated by varying the anomaly threshold over all tissue regions' anomalous scores. To compute the F1-score, a threshold is first identified such that the number of regions exceeding it matches the true number of anomalous regions, after which the F1-score is computed for regions whose scores are above this threshold.

\section{Algorithmic workflow of \textit{OMIB}} \label{pesudo code}

\vspace{-1cm}
\renewcommand{\algorithmicrequire}{\textbf{Input:}}
\renewcommand{\algorithmicensure}{\textbf{Output:}}
{
\begin{center}
  \begin{tabular}{p{\textwidth}}
    \begin{algorithm}[H]
      \caption{Warm-up training}
      \begin{algorithmic}[1]
        \Require
        Modality \( v_k \), \( k \in \{1, 2\} \), Maximum epochs $E_{max}$, Batch size $N$.
        \vspace{0pt}
        \renewcommand{\algorithmicrequire}{\textbf{Notation:}}
        \Require
        \( Enc_k \): Unimodal encoder for modality \( v_k \);
        \( Dec_k \): Task-relevant prediction head for modality \( v_k \);
        \( z_k \): Latent representation of modality \( v_k \);
        \( e_k \): Stochastic Gaussian noises;
        \vspace{0pt}
        \Ensure
        \( Enc_k \) and \( Dec_k \).
        \vspace{0pt}
        \State Initialize \( Enc_k \) and \( Dec_k \), $\forall k \in \{1,2\}$;
         \While{$epoch<E_{max}$}
            \State Sample a batch \(\{v_k^i \mid i \in \{1, 2, \dots, N\}\}\) from each modality \(k \in \{1, 2\}\);
            \For{each \( i \in \{1, 2, \dots, N\} \)}
                \For{each modality \( k \in \{1, 2\} \)}
                    \State \( z_k^i = Enc_k^i(v_k^i) \);
                    \State \( e_k^i \sim \mathcal{N}(0, I) \); 
                    \State \( \hat{y}_k^i = Dec_k([z_k^i, e_k^i]) \);
                \EndFor
            \EndFor
            \State  Compute \( L_{TRB_k} \) as in \Cref{loss TRB} for each modality \( k \in \{1, 2\} \);
            \State Update \( Enc_k \) and \( Dec_k \) using gradient descent; 
        \EndWhile \\
        \Return \( Enc_k \) and \( Dec_k \)
      \end{algorithmic}
    \end{algorithm}\\
  \end{tabular}
  \label{code:I}
\end{center}
}
\vspace{-2cm}

\renewcommand{\algorithmicrequire}{\textbf{Input:}}
\renewcommand{\algorithmicensure}{\textbf{Output:}}
{
\begin{center}
  \begin{tabular}{p{\textwidth}}
    \begin{algorithm}[H]
     \caption{Main training}
      \begin{algorithmic}[1]
        \Require
        Modality \( v_k \), Unimodal encoder\( Enc_k \), Task-relevant prediction head \( Dec_k \), \(\forall k \in \{1, 2\} \), Maximum epochs $E_{max}$, Batch size $N$.
        \vspace{0pt}
        \renewcommand{\algorithmicrequire}{\textbf{Notation:}}
        \Require
        \( VAE_k \): Variational encoder for modality \( v_k \);
        \( \zeta_k \): Latent representation of modality \( v_k \) after reparameterization;
        \( CAN \): Cross-attention network;
        \( \widehat{Dec} \): OMF task-relevant prediction head;
        \( \text{MINE} \): Mutual Information Neural Estimation (MINE); 
        \( \epsilon_k \): Standard Gaussian samples.
        \vspace{0pt}
        \Ensure
        \( Enc_k \), \( \forall k \in \{1, 2\} \), OMF (\( VAE_k \), \( \forall k \in \{1, 2\} \), \( CAN \), and \( \widehat{Dec} \)).
        \vspace{0pt}
        \For{each modality $k \in \{1, 2\}$}
            \State \(H(v_k)=\text{MINE}(v_k,v_k)\);
        \EndFor
        \State \(I(v_1;v_2)=\text{MINE}(v_1,v_2)\);
        \State Sample $\beta$ from the range $[M_l,M_u]$, where $M_l:=\frac{1}{3(H(v_1)+H(v_2))}$, $M_u:=\frac{1}{3(H(v_1)+H(v_2)-I(v_1;v_2))}$;
        \While{$epoch<E_{max}$}
            \State Sample a batch \(\{v_k^i \mid i \in \{1, 2, \dots, N\}\}\) from each modality \(k \in \{1, 2\}\);
            \For{each \( i \in \{1, 2, \dots, N\} \)}
                \For{each modality \( k \in \{1, 2\} \)}
                    \State \( z_k^i = Enc_k(v_k^i) \); 
                    \State \( \mu_k^i, \Sigma_k^i = VAE_i(z_k^i) \); 
                    \State \( \zeta_k^i = \mu_k^i + \Sigma_k^i \times \epsilon_i \); 
                \EndFor
                \State \( \xi^i = CAN(\zeta_1^i, \zeta_2^i) \);
                \For{each modality \( i \in \{1, 2\} \)}
                    \State \( \hat{y}_k^i = Dec_i([z_k^i, \xi^i]) \);
                \EndFor
                \State \( \hat{y}^i = \widehat{Dec}(\xi^i) \);
                \State Adjust $r$ as defined in \Cref{equ:r};
            \EndFor 
            \State Compute \( L_{OMF} \) as in \Cref{equ:IB loss imp}, and \( L_{TRB_k} \) as in \Cref{loss TRB} for each modality \( i \in \{1, 2\} \);
            \State $L = L_{OMF} + L_{TRB_1} + L_{TRB_2}$;
            \State Update parameters of \( Enc_k \), \( VAE_k \), \( CAN \), \( Dec_k \), and \( \widehat{Dec} \) using gradient descent;
        \EndWhile
        \vspace{0pt} \\
        \Return \( Enc_k \), \( \forall k \in \{1, 2\} \), OMF (\( VAE_k \), \( \forall k \in \{1, 2\} \), \( CAN \), and \( \widehat{Dec} \))
      \end{algorithmic}
    \end{algorithm}\\
  \end{tabular}
  \label{code:II}
\end{center}
}



\end{document}